%% file: main.tex
\documentclass[final]{IEEEtran}
\input{commands}
\graphicspath{{pix/}}
\usepackage[hmargin={1.8cm}, vmargin={2cm}]{geometry}

\title{Adaptive Image Denoising by Targeted Databases}
\author{Enming Luo,~\IEEEmembership{Student Member,~IEEE}, Stanley~H.~Chan,~\IEEEmembership{Member,~IEEE}, and Truong~Q.~Nguyen,~\IEEEmembership{Fellow,~IEEE}
\thanks{E. Luo and T. Nguyen are with Department of Electrical and Computer Engineering, University of California at San Diego, La Jolla, CA 92093, USA. Emails: eluo@ucsd.edu and nguyent@ece.ucsd.edu}
\thanks{S. Chan is with School of Electrical and Computer Engineering, Purdue University, West Lafayette, IN 47907, USA. Email: stanleychan@purdue.edu}
\thanks{This work was supported in part by a Croucher Foundation Post-doctoral Research Fellowship, and in part by the National Science Foundation under grant CCF-1065305. Preliminary material in this paper was presented at the 39th IEEE International Conference on Acoustics, Speech and Signal Processing (ICASSP), Florence, May 2014. }
}

\graphicspath{{pix/}}

\begin{document}
\maketitle

\begin{abstract}
We propose a data-dependent denoising procedure to restore noisy images. Different from existing denoising algorithms which search for patches from either the noisy image or a generic database, the new algorithm finds patches from a database that contains relevant patches. We formulate the denoising problem as an optimal filter design problem and make two contributions. First, we determine the basis function of the denoising filter by solving a group sparsity minimization problem. The optimization formulation generalizes existing denoising algorithms and offers systematic analysis of the performance. Improvement methods are proposed to enhance the patch search process. Second, we determine the spectral coefficients of the denoising filter by considering a localized Bayesian prior. The localized prior leverages the similarity of the targeted database, alleviates the intensive Bayesian computation, and links the new method to the classical linear minimum mean squared error estimation. We demonstrate applications of the proposed method in a variety of scenarios, including text images, multiview images and face images. Experimental results show the superiority of the new algorithm over existing methods.
\end{abstract}


\input{introduction}

\input{problem_setup}

\input{proposed_method_part1}
\input{proposed_method_part2}

\input{experiments}

\input{conclusion}

\input{appendix}

\bibliographystyle{IEEEbib}
\bibliography{ref_external}
\end{document}

%% file: commands.tex
\usepackage[noadjust]{cite}

\usepackage[dvips]{graphicx}
\usepackage{multirow,multicol}
\usepackage[table]{xcolor}
\usepackage{ctable}

\usepackage[tbtags]{amsmath}
\usepackage{amsbsy}
\usepackage{amssymb}
\usepackage{amsfonts}
\usepackage{dsfont}

\usepackage{graphicx}
\usepackage{subfig}
\usepackage{url}
\usepackage{textcomp}
\usepackage{algorithmic}
\usepackage{algorithm}
\usepackage{balance}
\usepackage{rotating}
\usepackage{caption}
\usepackage{fancyvrb}
\usepackage{latexsym}
\usepackage{verbatim}

\usepackage{etex}
\usepackage{tikz}
\usepackage{pgfplots}

\newtheorem{lemma}{Lemma}

\newtheorem{remark}{Remark}

{\begin{list}               
    {$\bullet$ \hfill}{
        \setlength{\leftmargin}{\parindent}
        \setlength{\parsep}{0.04\baselineskip}
        \setlength{\itemsep}{0.5\parsep}
        \setlength{\labelwidth}{\leftmargin}
        \setlength{\labelsep}{0em}}
    }
{\end{list}}

\providecommand{\eref}[1]{\eqref{#1}}  
\providecommand{\cref}[1]{Chapter~\ref{#1}}

\providecommand{\fref}[1]{Figure~\ref{#1}}

\providecommand{\R}{\ensuremath{\mathbb{R}}}

\providecommand{\E}{\ensuremath{\mathbb{E}}}

\providecommand{\bydef}{\overset{\text{def}}{=}}

\renewcommand{\vec}[1]{\ensuremath{\boldsymbol{#1}}}
\providecommand{\mat}[1]{\ensuremath{\boldsymbol{#1}}}

\providecommand{\calL}{\mathcal{L}}

\providecommand{\calN}{\mathcal{N}}

\providecommand{\calP}{\mathcal{P}}

\providecommand{\calS}{\mathcal{S}}

\providecommand{\mA}{\mat{A}}
\providecommand{\mB}{\mat{B}}

\providecommand{\mG}{\mat{G}}

\providecommand{\mI}{\mat{I}}

\providecommand{\mP}{\mat{P}}

\providecommand{\mS}{\mat{S}}

\providecommand{\mU}{\mat{U}}

\providecommand{\mW}{\mat{W}}
\providecommand{\mX}{\mat{X}}

\providecommand{\vc}{\vec{c}}
\providecommand{\ve}{\vec{e}}

\providecommand{\vp}{\vec{p}}
\providecommand{\vq}{\vec{q}}

\providecommand{\vu}{\vec{u}}

\providecommand{\vx}{\vec{x}}

\providecommand{\mLambda}{\mat{\Lambda}}

\providecommand{\mSigma}{\mat{\Sigma}}

\providecommand{\w}{\omega}

\providecommand{\veta}{\vec{\eta}}

\providecommand{\vmu}{\vec{\mu}}

\providecommand{\vphat}{\boldsymbol{\widehat{p}}}
\providecommand{\vpbar}{\overline{\boldsymbol{p}}}

\providecommand{\vzero}{\vec{0}}
\providecommand{\vone}{\vec{1}}

\newcommand{\defequal}{\mathop{\overset{\mbox{\tiny{def}}}{=}}}
\newcommand{\subjectto}{\mathop{\mathrm{subject\, to}}}
\newcommand{\argmin}[1]{\mathop{\underset{#1}{\mathrm{arg\,min}}}}

\newcommand{\minimize}[1]{\mathop{\underset{#1}{\mathrm{minimize}}}}

\newcommand{\diag}[1]{\mathop{\mathrm{diag}\left\{#1\right\}}}


%% file: introduction.tex
\section{Introduction}
\subsection{Patch-based Denoising}
Image denoising is a classical signal recovery problem where the goal is to restore a clean image from its observations. Although image denoising has been studied for decades, the problem remains a fundamental one as it is the test bed for a variety of image processing tasks.

Among the numerous contributions in image denoising in the literature, the most highly-regarded class of methods, to date, is the class of $\emph{patch-based image denoising}$ algorithms \cite{Buades_Coll_2005_Journal, Kervrann_Boulanger_2007, Dabov_Foi_Katkovnik_2007, Dabov_Foi_Katkovnik_Egiazarian_2008, Dabov_Foi_Katkovnik_Egiazarian_2009, Mairal_Bach_Ponce_Sapiro_Zisserman_2009, Zhang_Dong_Zhang_2010, Dong_Zhang_2013, Rajwade_Rangarajan_Banerjee_2013}. The idea of a patch-based denoising algorithm is simple: Given a $\sqrt{d}\times\sqrt{d}$ patch $\vq\in \R^d$ from the noisy image, the algorithm finds a set of reference patches $\vp_1,\ldots,\vp_k \in \R^d$ and applies some linear~(or non-linear) function $\Phi$ to obtain an estimate $\vphat$ of the unknown clean patch $\vp$ as
\begin{equation}
\vphat = \Phi(\vq;\, \vp_1,\ldots,\vp_k).
\label{eq:patch-based denoising}
\end{equation}
For example, in non-local means~(NLM)~\cite{Buades_Coll_2005_Journal}, $\Phi$ is a weighted average of the reference patches, whereas in BM3D~\cite{Dabov_Foi_Katkovnik_2007}, $\Phi$ is a transform-shrinkage operation.

\subsection{Internal vs External Denoising}
For any patch-based denoising algorithm, the denoising performance is intimately related to the reference patches $\vp_1,\ldots,\vp_k$. Typically, there are two sources of these patches: the noisy image itself and an external database of patches. The former is known as \emph{internal denoising} \cite{Zontak_Irani_2011}, whereas the latter is known as \emph{external denoising} \cite{Mosseri_Zontak_Irani_2013, Burger_Schuler_harmeling_2013}.

Internal denoising is practically more popular than external denoising because it is computationally less expensive. Moreover, internal denoising does not require a training stage, hence making it free of training bias. Furthermore, Glasner et al. \cite{Glasner_Bagon_Irani_2009} showed that patches tend to recur within an image, \emph{e.g.}, at a different location, orientation, or scale. Thus searching for patches in the noisy image is often a plausible approach. However, on the downside, internal denoising often fails for rare patches --- patches that seldom recur in an image. This phenomenon is known as ``rare patch effect'', and is widely regarded as a bottleneck of internal denoising \cite{Chatterjee_Milanfar_2010,Levin_Nadler_2011}. There are some works \cite{Yan_Shao_Cvetkovic_Klijn_2012,Lou_Favaro_Soatto_Bertozzi_2009} attempting to alleviate the rare patch problem. However, the extent to which these methods can achieve is still limited.

External denoising \cite{Mairal_Bach_Ponce_Sapiro_Zisserman_2009,Elad_Aharon_2006,Zoran_Weiss_2011,Chan_Zickler_Lu_2013,Chan_Zickler_Lu_2014} is an alternative solution to internal denoising. Levin et al. \cite{Levin_Nadler_2011, Levin_Nadler_Durand_Freeman_2012} showed that in the limit, the theoretical minimum mean squared error of denoising is achievable using an infinitely large external database. Recently, Chan et al. \cite{Chan_Zickler_Lu_2013,Chan_Zickler_Lu_2014} developed a computationally efficient sampling scheme to reduce the complexity and demonstrated practical usage of large databases. However, in most of the recent works on external denoising, \emph{e.g.}, \cite{Mairal_Bach_Ponce_Sapiro_Zisserman_2009,Elad_Aharon_2006,Zoran_Weiss_2011}, the databases used are \emph{generic}. These databases, although large in volume, do not necessarily contain useful information to denoise the noisy image of interest. For example, it is clear that a database of natural images is not helpful to denoise a noisy portrait image.

\subsection{Adaptive Image Denoising}
In this paper, we propose an adaptive image denoising algorithm using a \emph{targeted} external database instead of a \emph{generic} database. Here, a targeted database refers to a database that contains images \emph{relevant} to the noisy image only. As will be illustrated in later parts of this paper, targeted external databases could be obtained in many practical scenarios, such as text images (\emph{e.g.}, newspapers and documents), human faces (under certain conditions), and images captured by multiview camera systems. Other possible scenarios include images of license plates, medical CT and MRI images, and images of landmarks.

The concept of using targeted external databases has been proposed in various occasions, \emph{e.g.}, \cite{Joshi_Matusik_Adelson_kriegman_2010, Sun_Hays_2012, Yang_Wright_Huang_Ma_2008, Johnson_Dale_Avidan_Pfister_Freeman_Matusik_2011, Elad_Datsenko_2009,Dale_Johnson_Sunkavalli_Matusik_Pfister_2009}. However, none of these methods are tailored for image denoising problems. The objective of this paper is to bridge the gap by addressing the following question:

\begin{center}
(Q): Suppose we are \emph{given} a targeted external database, how should we design a denoising algorithm which can \emph{maximally} utilize the database?
\end{center}
Here, we assume that the reference patches $\vp_1,\ldots,\vp_k$ are \emph{given}. We emphasize that this assumption is application specific --- for the examples we mentioned earlier (\emph{e.g.}, text, multiview, face, etc), the assumption is typically true because these images have relatively less variety in content.

When the reference patches are given, question (Q) may look trivial at the first glance because we can extend existing internal denoising algorithms in a brute-force way to handle external databases. For example, one can modify existing algorithms, \emph{e.g.}, \cite{Buades_Coll_2005_Journal,Dabov_Foi_Katkovnik_2007,Dabov_Foi_Katkovnik_Egiazarian_2009,Ram_Elad_Cohen_2013,Shao_Zhang_deHann_2008}, so that the patches are searched from a database instead of the noisy image. Likewise, one can also treat an external database as a ``video'' and feed the data to multi-image denoising algorithms, \emph{e.g.}, \cite{Dabov_Foi_Egiazarian_2007,Zhang_Vaddadi_Jin_Nayar_2009,Buades_Lou_Morel_Tang_2009,Luo_Chan_Pan_Nguyen_2013}. However, the problem of these approaches is that the brute force modifications are heuristic. There is no theoretical guarantee of performance. This suggests that a straight-forward modification of existing methods does \emph{not} solve question (Q), as the database is not maximally utilized.

An alternative response to question (Q) is to train a statistical prior of the targeted database, \emph{e.g.}, \cite{Mairal_Bach_Ponce_Sapiro_Zisserman_2009, Elad_Aharon_2006, Aharon_Elad_Bruckstein_2005, Roth_Black_2009, Zoran_Weiss_2011, Yu_Sapiro_Mallat_2012, Yan_Shao_Liu_2013}. The merit of this approach is that the performance often has theoretical guarantee because the denoising problem can now be formulated as a maximum a posteriori (MAP) estimation. However, the drawback is that many of these methods require a large number of training samples which is not always available in practice.

\subsection{Contributions and Organization}
In view of the above seemingly easy yet challenging question, we introduced a new denoising algorithm using targeted external databases in \cite{Luo_Chan_Nguyen_2014}. Compared to existing methods, the method proposed in \cite{Luo_Chan_Nguyen_2014} achieves better performance and only requires a small number of external images. In this paper, we extend \cite{Luo_Chan_Nguyen_2014} by offering the following new contributions:

\begin{enumerate}
  \item Generalization of Existing Methods. We propose a generalized framework which encapsulates a number of denoising algorithms. In particular, we show (in Section III-B) that the proposed group sparsity minimization generalizes both fixed basis and PCA methods. We also show (in Section IV-B) that the proposed local Bayesian MSE solution is a generalization of many spectral operations in existing methods.
  \item Improvement Strategies. We propose two improvement strategies for the generalized denoising framework. In Section III-D, we present a patch selection optimization to improve the patch search process. In Section IV-D, we present a soft-thresholding and a hard-thresholding method to improve the spectral coefficients learned by the algorithm.
  \item Detailed Proofs. Proofs of the results in this paper and \cite{Luo_Chan_Nguyen_2014} are presented in the Appendix.
\end{enumerate}

The rest of the paper is organized as follows. After outlining the design framework in Section II, we present the above contributions in Section III -- IV. Experimental results are discussed in Section V, and concluding remarks are given in Section VI.

%% file: problem_setup.tex
\section{Optimal Linear Denoising Filter}
\label{section2}
The foundation of our proposed method is the classical optimal linear denoising filter design problem \cite{Milanfar_2013a}. In this section, we give a brief review of the design framework and highlight its limitations.

\subsection{Optimal Filter}
The design of an optimal denoising filter can be posed as follows: Given a noisy patch $\vq \in \R^d$, and assuming that the noise is i.i.d. Gaussian with zero mean and variance $\sigma^2$, we want to find a linear operator $\mA \in \R^{d \times d}$ such that the estimate $\vphat = \mA \vq$ has the minimum mean squared error (MSE) compared to the ground truth $\vp \in \R^d$. That is, we want to solve the optimization
\begin{equation}
\mA = \argmin{\mA}\;\E\left[\|\mA\vq - \vp\|_2^2\right].
\label{eq:mse problem}
\end{equation}

Here, we assume that $\mA$ is symmetric, or otherwise the Sinkhorn-Knopp iteration \cite{Milanfar_2013b} can be used to symmetrize $\mA$, provided that entries of $\mA$ are non-negative. Given a symmetric $\mA$, one can apply the eigen-decomposition, $\mA = \mU\mLambda\mU^T$, where $\mU = [\vu_1,\ldots,\vu_d] \in \R^{d \times d}$ is the basis matrix and $\mLambda = \diag{\lambda_1,\ldots,\lambda_d}\in \R^{d \times d}$ is the diagonal matrix containing the spectral coefficients. With $\mU$ and $\mLambda$, the optimization problem in \eref{eq:mse problem} becomes
\begin{equation}
(\mU,\mLambda) = \argmin{\mU,\mLambda}\;\E\left[ \left\| \mU\mLambda\mU^T\vq - \vp \right\|_2^2\right],
\label{eq:mse_find U Lambda}
\end{equation}
subject to the constraint that $\mU$ is an orthonormal matrix.

The joint optimization \eref{eq:mse_find U Lambda} can be solved by noting the following Lemma.
\begin{lemma}
\label{lemma:mse}
Let $\vu_i$ be the $i$th column of the matrix $\mU$, and $\lambda_i$ be the $(i,i)$th entry of the diagonal matrix $\mLambda$. If $\vq = \vp + \veta$, where $\veta \overset{\mathrm{\scriptsize{iid}}}{\sim} \calN(\vzero,\sigma^2\mI)$, then
\begin{equation}
\E\left[ \left\| \mU\mLambda\mU^T\vq - \vp \right\|_2^2\right]
= \sum_{i=1}^d \left[(1-\lambda_i)^2 (\vu_i^T\vp)^2 + \sigma^2 \lambda_i^2\right].
\label{eq:mse}
\end{equation}
\end{lemma}

The proof of Lemma~\ref{lemma:mse} is given in \cite{Talebi_Milanfar_2014}. With Lemma~\ref{lemma:mse}, the denoised patch as a consequence of \eref{eq:mse_find U Lambda} is as follows.
\begin{lemma}
\label{lemma:mse,solution}
The denoised patch $\vphat$ using the optimal $\mU$ and $\mLambda$ of \eref{eq:mse_find U Lambda} is
\begin{equation*}
\vphat = \mU
\left(\diag{\frac{\|\vp\|^2}{\|\vp\|^2+ \sigma^2},0,\ldots,0}\right)
\mU^T\vq,
\label{eq:mse,phat1}
\end{equation*}
where $\mU$ is any orthonormal matrix with the first column $\vu_1 = \vp/\|\vp\|_2$.
\end{lemma}
\begin{proof}
See Appendix \ref{appendix:proof,mse,solution}.
\end{proof}

Lemma~\ref{lemma:mse,solution} states that if hypothetically we are given the ground truth $\vp$, the optimal denoising process is to first project the noisy observation $\vq$ onto the subspace spanned by $\vp$, then perform a Wiener shrinkage $\|\vp\|^2/ (\|\vp\|^2+\sigma^2)$, and finally re-project the shrinkage coefficients to obtain the denoised estimate. However, since in reality we never have access to the ground truth $\vp$, this optimal result is not achievable.

\subsection{Problem Statement}
Since the oracle optimal filter is not achievable in practice, the question becomes whether it is possible to find a surrogate solution that does not require the ground truth $\vp$.

To answer this question, it is helpful to separate the joint optimization \eref{eq:mse_find U Lambda} by first fixing $\mU$ and minimize the MSE with respect to $\mLambda$. In this case, one can show that \eref{eq:mse} achieves the minimum when
\begin{equation}
\lambda_i = \frac{(\vu_i^T\vp)^2}{(\vu_i^T\vp)^2 + \sigma^2},
\label{eq:mse,lambda_i}
\end{equation}
in which the minimum MSE estimator is given by
\begin{equation}
\vphat = \mU
\left(\diag{\frac{(\vu_1^T\vp)^2}{(\vu_1^T\vp)^2+ \sigma^2},\ldots,\frac{(\vu_d^T\vp)^2}{(\vu_d^T\vp)^2+ \sigma^2}}\right)
\mU^T\vq,
\label{eq:mse,phat}
\end{equation}
where $\{\vu_1,\ldots,\vu_d\}$ are the columns of $\mU$.

Inspecting \eref{eq:mse,phat}, we identify two parts of the problem:
\begin{enumerate}
\item Determine $\mU$. The choice of $\mU$ plays a critical role in the denoising performance. In literature, $\mU$ are typically chosen as the FFT or the DCT bases \cite{Dabov_Foi_Katkovnik_2007, Dabov_Foi_Katkovnik_Egiazarian_2008}. In \cite{Dabov_Foi_Katkovnik_Egiazarian_2009, Zhang_Dong_Zhang_2010, Dong_Zhang_2013}, the PCA bases of various data matrices are proposed. However, the optimality of these bases is not fully understood.
\item Determine $\mLambda$. Even if $\mU$ is fixed, the optimal $\mLambda$ in \eref{eq:mse,lambda_i} still depends on the unknown ground truth $\vp$. In \cite{Dabov_Foi_Katkovnik_2007}, $\mLambda$ is determined by hard-thresholding a stack of DCT coefficients or applying an empirical Wiener filter constructed from a first-pass estimate. In \cite{Zhang_Dong_Zhang_2010}, $\mLambda$ is formed by the PCA coefficients of a set of relevant noisy patches. Again, it is unclear which of these is optimal.
\end{enumerate}

Motivated by the problems about $\mU$ and $\mLambda$, in the following two sections we present our proposed method for each of these problems. We discuss its relationship to prior works, and present ways to further improve it.

%% file: proposed_method_part1.tex
\section{Determine $\mU$}
\label{section3}
In this section, we present our proposed method to determine the basis matrix $\mU$ and show that it is a generalization of a number of existing denoising algorithms. We also discuss ways to improve $\mU$.

\subsection{Patch Selection via $k$ Nearest Neighbors}
Given a noisy patch $\vq$ and a targeted database $\{\vp_j\}_{j=1}^n$, our first task is to fetch the $k$ most ``relevant'' patches. The patch selection is performed by measuring the similarity between $\vq$ and each of $\{\vp_j\}_{j=1}^n$, defined as
\begin{equation}
d(\vq,\vp_j) = \|\vq - \vp_j\|_2, \quad \mbox{ for } j = 1,\ldots,n.
\label{eq:knn}
\end{equation}
We note that \eref{eq:knn} is equivalent to the standard $k$ nearest neighbors ($k$NN) search.

$k$NN has a drawback that under the $\ell_2$ distance, some of the $k$ selected patches may not be truly relevant to the denoising task, because the query patch $\vq$ is noisy. We will come back to this issue in Section III-D by discussing methods to improve the robustness of the $k$NN.

\subsection{Group Sparsity}
\label{section3:group sparsity}
Without loss of generality, we assume that the $k$NN returned by the above procedure are the first $k$ patches of the data, i.e., $\{\vp_j\}_{j=1}^k$. Our goal now is to construct $\mU$ from $\{\vp_j\}_{j=1}^k$.

We postulate that a good $\mU$ should have two properties. First, $\mU$ should make the projected vectors $\{\mU^T\vp_j\}_{j=1}^k$ similar in both \emph{magnitude} and \emph{location}. This hypothesis follows from the observation that since $\{\vp_j\}_{j=1}^k$ have small $\ell_2$ distances from $\vq$, it must hold that any $\vp_i$ and $\vp_j$ (hence $\mU^T\vp_i$ and $\mU^T\vp_j$) in the set should also be similar. Second, we require that each projected vector $\mU^T\vp_j$ contains as few non-zeros as possible, \emph{i.e.}, \emph{sparse}. The reason is related to the shrinkage step to be discussed in Section \ref{section4}, because a vector of few non-zero coefficients has higher energy concentration and hence is more effective for denoising.

In order to satisfy these two criteria, we propose to consider the idea of \emph{group sparsity}\footnote{Group sparsity was first proposed by Cotter et al. for group sparse reconstruction \cite{Cotter_Rao_Engan_Kreutz_2005} and later used by Mairal et al. for denoising \cite{Mairal_Bach_Ponce_Sapiro_Zisserman_2009}, but towards a different end from the method presented in this paper.}, which is characterized by the matrix $\ell_{1,2}$ norm, defined as \footnote{In general one can define $\ell_{p,q}$ norm as $\|\mX\|_{p,q} = \sum_{i=1}^d \|\vx_i\|_q^p$, c.f. \cite{Mairal_Bach_Ponce_Sapiro_Zisserman_2009}.}
\begin{equation}
\| \mX \|_{1,2} \bydef \sum_{i=1}^d \| \vx_i \|_2,
\end{equation}
for any matrix $\mX \in \R^{d \times k}$, where $\vx_i \in \R^k$ is the $i$th row of a matrix $\mX$. In words, a small $\| \mX \|_{1,2}$ makes sure that $\mX$ has few non-zero entries, and the non-zero entries are located similarly in each column \cite{Cotter_Rao_Engan_Kreutz_2005,Mairal_Bach_Ponce_Sapiro_Zisserman_2009}. A pictorial illustration is shown in \fref{fig:sparse vs group sparse}.

\input{sparse_vs_group_sparse}

Going back to our problem, we propose to minimize the $\ell_{1,2}$-norm of the matrix $\mU^T\mP$:
\begin{equation}
\begin{array}{ll}
\minimize{\mU}  &\quad \|\mU^T\mP\|_{1,2} \\
\subjectto      &\quad \mU^T\mU = \mI,
\end{array}
\label{eq:determine U,problem}
\end{equation}
where $\mP \bydef [\vp_1,\ldots,\vp_k]$. The equality constraint in \eref{eq:determine U,problem} ensures that $\mU$ is orthonormal. Thus, the solution of \eref{eq:determine U,problem} is an orthonormal matrix $\mU$ which maximizes the group sparsity of the data $\mP$.

Interestingly, and surprisingly, the solution of \eref{eq:determine U,problem} is indeed \emph{identical} to the classical principal component analysis (PCA). The following lemma summarizes the observation.

\begin{lemma}
\label{lemma:determine U,solution}
The solution to \eref{eq:determine U,problem} is that
\begin{equation}
[\mU,\mS] = \mbox{eig}(\mP\mP^T),
\label{eq:determine U,solution}
\end{equation}
where $\mS$ is the corresponding eigenvalue matrix.
\end{lemma}
\begin{proof}
See Appendix \ref{appendix:proof,determine U,solution}.
\end{proof}

\begin{remark}
In practice, it is possible to improve the fidelity of the data matrix $\mP$ by introducing a diagonal weight matrix
\begin{equation}
\mW = \frac{1}{Z} \diag{e^{-\|\vq - \vp_1\|^2/h^2},\ldots,e^{-\|\vq - \vp_k\|^2/h^2}},
\label{eq:W}
\end{equation}
for some user tunable parameter $h$ and a normalization constant $Z \bydef \vone^T\mW\vone$. Consequently, we can define
\begin{equation}
\overline{\mP} = \mP\mW^{1/2}.
\label{eq:PW}
\end{equation}
Hence \eref{eq:determine U,solution} becomes $[\mU, \mS] = \mbox{eig}(\mP\mW\mP^T)$.
\end{remark}

\subsection{Relationship to Prior Works}
The fact that \eref{eq:determine U,solution} is the solution to a group sparsity minimization problem allows us to understand the performance of a number of existing denoising algorithms to some extent.

\subsubsection{BM3D \cite{Dabov_Foi_Katkovnik_2007}}
It is perhaps a misconception that the underlying principle of BM3D is to enforce sparsity of the 3-dimensional data volume (which we shall call it a 3-way tensor). However, what BM3D enforces is the \emph{group sparsity} of the slices of the tensor, not the sparsity of the tensor.

To see this, we note that the 3-dimensional transforms in BM3D are separable (\emph{e.g.}, DCT2 + Haar in its default setting). If the patches $\vp_1,\ldots,\vp_k$ are sufficiently similar, the DCT2 coefficients will be similar in \emph{both} magnitude and location \footnote{By DCT2 location we meant the frequency of the DCT2 components.}. Therefore, by fixing the frequency location of a DCT2 coefficient and tracing the DCT2 coefficients along the third axis, the output signal will be almost flat. Hence, the final Haar transform will return a sparse vector. Clearly, such sparsity is based on the stationarity of the DCT2 coefficients along the third axis. In essence, this is group sparsity.

\subsubsection{HOSVD \cite{Rajwade_Rangarajan_Banerjee_2013}}
The true tensor sparsity can only be utilized by the high order singular value decomposition (HOSVD), which is recently studied in \cite{Rajwade_Rangarajan_Banerjee_2013}. Let $\calP \in \R^{\sqrt{d}\times\sqrt{d}\times k}$ be the tensor by stacking the patches $\vp_1,\ldots,\vp_k$ into a 3-dimensional array, HOSVD seeks three orthonormal matrices $\mU^{(1)} \in \R^{\sqrt{d}\times\sqrt{d}}$, $\mU^{(2)} \in \R^{\sqrt{d}\times\sqrt{d}}$, $\mU^{(3)} \in \R^{k \times k}$ and an array $\calS \in \R^{\sqrt{d}\times\sqrt{d}\times k}$, such that
\begin{equation*}
\calS = \calP \times_1 \mU^{(1)^T} \times_2 \mU^{(2)^T} \times_3 \mU^{(3)^T},
\end{equation*}
where $\times_k$ denotes a tensor mode-$k$ multiplication \cite{Kolda_Bader_2009}.

As reported in \cite{Rajwade_Rangarajan_Banerjee_2013}, the performance of HOSVD is indeed worse than BM3D. This phenomenon can now be explained, because HOSVD ignores the fact that image patches tend to be group sparse instead of being tensor sparse.

\subsubsection{Shape-adaptive BM3D \cite{Dabov_Foi_Katkovnik_Egiazarian_2008}}
As a variation of BM3D, SA-BM3D groups similar patches according to a shape-adaptive mask. Under our proposed framework, this shape-adaptive mask can be modeled as a spatial weight matrix $\mW_s \in \R^{d \times d}$ (where the subscript $s$ denotes \emph{spatial}). Adding $\mW_s$ to \eref{eq:PW}, we define
\begin{equation}
\overline{\mP} = \mW_s^{1/2} \mP \mW^{1/2}.
\end{equation}
Consequently, the PCA of $\overline{\mP}$ is equivalent to SA-BM3D. Here, the matrix $\mW_s$ is used to control the relative emphasis of each pixel in the spatial coordinate.

\subsubsection{BM3D-PCA \cite{Dabov_Foi_Katkovnik_Egiazarian_2009} and LPG-PCA \cite{Zhang_Dong_Zhang_2010}}
The idea of both BM3D-PCA and LPG-PCA is that given $\vp_1,\ldots,\vp_k$, $\mU$ is determined as the principal components of $\mP = [\vp_1,\ldots,\vp_k]$. Incidentally, such approaches arrive at the same result as \eref{eq:determine U,solution}, \emph{i.e.}, the principal components are indeed the solution of a group sparse minimization. However, the key of using the group sparsity is not noticed in  \cite{Dabov_Foi_Katkovnik_Egiazarian_2009} and \cite{Zhang_Dong_Zhang_2010}. This provides additional theoretical justifications for both methods.

\subsubsection{KSVD \cite{Elad_Aharon_2006}} In KSVD, the dictionary plays the role of our basis matrix $\mU$. The dictionary can be trained either from the single noisy image, or from an external (generic or targeted) database. However, the training is performed once for \emph{all} patches of the image. In other words, the noisy patches share a \emph{common} dictionary. In our proposed method, \emph{each} noisy patch has an individually trained basis matrix. Clearly, the latter approach, while computationally more expensive, is significantly more data adaptive than KSVD.

\begin{figure*}[th]
\centering
\def\iw{0.3\linewidth}
\begin{tabular}{cccc}
\hspace{-8ex}\includegraphics[width=\iw]{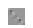}
&\hspace{-15ex}\includegraphics[width=\iw]{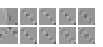}
&\hspace{-4ex} \includegraphics[width=\iw]{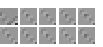}
&\hspace{-4ex} \includegraphics[width=\iw]{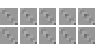}\\
\hspace{-12ex}(a) $\vp$& \hspace{-15ex}(b) $\varphi(\vx)=0$& \hspace{-4ex}(c) $\varphi(\vx)=\vone^T\mB\vx$& \hspace{-4ex}(d) $\varphi(\vx)=\ve^T\vx$\\
\end{tabular}
\caption{Refined patch matching results: (a) ground truth, (b) 10 best reference patches using $\vq$ ($\sigma=50$), (c) 10 best reference patches using $\varphi(\vx)=\vone^T\mB\vx$ (where $\tau = 1/(2n)$), (d) 10 best reference patches using $\varphi(\vx)=\ve^T\vx$ (where $\tau = 1$).}
\label{fig:improved_patch_matching}
\end{figure*}

\subsection{Improvement: Patch Selection Refinement}
\label{improved U}
The optimization problem \eref{eq:determine U,problem} suggests that the $\mU$ computed from \eref{eq:determine U,solution} is the optimal basis with respect to the reference patches $\{\vp_j\}_{j=1}^k$. However, one issue that remains is how to improve the selection of $k$ patches from the original $n$ patches. Our proposed approach is to formulate the patch selection as an optimization problem
\begin{equation}
\begin{array}{ll}
\minimize{\vx}  &\quad \vc^T\vx + \tau \varphi(\vx)\\
\subjectto      &\quad \vx^T\vec{1} = k, \quad 0 \leq \vx \leq 1,
\end{array}
\label{eq:l2-based knn, penalty}
\end{equation}
where $\vc = [c_1,\cdots,c_n]^T$ with $c_j \bydef \|\vq - \vp_j\|_2$, $\varphi(\vx)$ is a penalty function and $\tau > 0$ is a parameter. In \eref{eq:l2-based knn, penalty}, each $c_j$ is the distance $\|\vq - \vp_j\|_2$, and $x_j$ is a weight indicating the emphasis of $\|\vq - \vp_j\|_2$. Therefore, the minimizer of \eref{eq:l2-based knn, penalty} is a sequence of weights that minimize the overall distance.

To gain more insight into \eref{eq:l2-based knn, penalty}, we first consider the special case when the penalty term $\varphi(\vx) = 0$. We claim that, under this special condition, the solution of \eref{eq:l2-based knn, penalty} is equivalent to the original $k$NN solution in \eref{eq:knn}. This result is important, because $k$NN is a fundamental building block of all patch-based denoising algorithms. By linking $k$NN to the optimization formulation in \eref{eq:l2-based knn, penalty} we provide a systematic strategy to improve the $k$NN.

The proof of the equivalence between $k$NN and \eref{eq:l2-based knn, penalty} can be understood via the following case study where $n=2$ and $k=1$. In this case, the constraints $\vx^T\vec{1} = 1$ and $0 \le \vx \le 1$ form a closed line segment in the positive quadrant. Since the objective function $\vc^T\vx$ is linear, the optimal point must be at one of the vertices of the line segment, which is either $\vx = [0,1]^T$, or $\vx = [1,0]^T$. Thus, by checking which of $c_1$ or $c_2$ is smaller, we can determine the optimal solution by setting $x_1 = 1$ if $c_1$ is smaller (and vice versa). Correspondingly, if $x_1=1$, then the first patch $\vp_1$ should be selected. Clearly, the solution returned by the optimization is exactly the $k$NN solution. A similar argument holds for higher dimensions, hence justifies our claim.

Knowing that $k$NN can be formulated as \eref{eq:l2-based knn, penalty}, our next task is to choose an appropriate penalty term. The following are two possible choices.

\subsubsection{Regularization by Cross Similarity}
The first choice of $\varphi(\vx)$ is to consider $\varphi(\vx) = \vx^T\mB \vx$, where $\mB \in \R^{n \times n}$ is a symmetric matrix with $B_{ij}\defequal \|\vp_i-\vp_j\|_2$. Writing \eref{eq:l2-based knn, penalty} explicitly, we see that the optimization problem \eref{eq:l2-based knn, penalty} becomes
\begin{equation}
\minimize{0 \le \vx \le 1,\, \vx^T\vone = k} \;\; \sum_{j} x_j \|\vq - \vp_j\|_2 + \tau \sum_{i,j} x_ix_j \|\vp_i - \vp_j\|_2.
\label{eq:l2-based knn, penalty1}
\end{equation}
The penalized problem \eref{eq:l2-based knn, penalty1} suggests that the optimal $k$ reference patches should not be determined merely from $\|\vq - \vp_j\|_2$ (which could be problematic due to the noise present in $\vq$). Instead, a good reference patch should also be similar to all other patches that are selected. The cross similarity term $x_i x_j \|\vp_i - \vp_j\|_2$ provides a way for such measure. This shares some similarities to the patch ordering concept proposed by Cohen and Elad \cite{Ram_Elad_Cohen_2013}. The difference is that the patch ordering proposed in \cite{Ram_Elad_Cohen_2013} is a shortest path problem that tries to organize the noisy patches, whereas ours is to solve a regularized optimization.

Problem \eref{eq:l2-based knn, penalty1} is in general not convex because the matrix $\mB$ is not positive semidefinite. One way to relax the formulation is to consider $\varphi(\vx) = \vone^T\mB\vx$. Geometrically, the solution of using $\varphi(\vx) = \vone^T\mB\vx$ tends to identify patches that are close to the \emph{sum} of all other patches in the set. In many cases, this is similar to $\varphi(\vx) = \vx^T\mB \vx$ which finds patches that are similar to every \emph{individual} patch in the set. In practice, we find that the difference between $\varphi(\vx) = \vx^T\mB \vx$ and $\varphi(\vx) = \vone^T\mB\vx$ in the final denoising result (PSNR of the entire image) is marginal. Thus, for computational efficiency we choose $\varphi(\vx) = \vone^T\mB\vx$.

\subsubsection{Regularization by First-pass Estimate}
The second choice of $\varphi(\vx)$ is based on a \emph{first-pass estimate} $\vpbar$ using some denoising algorithms, for example, BM3D or the proposed method without this patch selection step. In this case, by defining $e_j \bydef \|\vpbar - \vp_j\|_2$ we consider the penalty function $\varphi(\vx) = \ve^T\vx$, where $\ve = [e_1, \cdots, e_n]^T$. This implies the following optimization problem
\begin{equation}
\begin{array}{ll}
\minimize{0 \le \vx \le 1,\, \vx^T\vone = k} \;\; \sum_{j} x_j \|\vq - \vp_j\|_2 + \tau \sum_{j} x_j \|\vpbar - \vp_j\|_2.
\end{array}
\label{eq:l2-based knn, penalty2}
\end{equation}
By identifying the objective of \eref{eq:l2-based knn, penalty2} as $(\vc+\tau \ve)^T\vx$, we observe that \eref{eq:l2-based knn, penalty2} can be solved in closed form by locating the $k$ smallest entries of the vector $\vc + \tau \ve$.

The interpretation of \eref{eq:l2-based knn, penalty2} is straight-forward: The linear combination of $\|\vq - \vp_j\|_2$ and $\|\vpbar - \vp_j\|_2$ shows a competition between the noisy patch $\vq$ and the first-pass estimate $\vpbar$. In most of the common scenarios, $\|\vq - \vp_j\|_2$ is preferred when noise level is low, whereas $\vpbar$ is preferred when noise level is high. This in turn requires a good choice of $\tau$. Empirically, we find that $\tau = 0.01$ when $\sigma < 30$ and $\tau = 1$ when $\sigma > 30$ is a good balance between the performance and generality.

\begin{figure}[t]
\vspace{-5ex}
\centering
\def\iw{0.3\linewidth}
\begin{tabular}{cccc}
              \includegraphics[trim = 0mm 0mm 0.5mm 0.5mm, clip, width=\iw]{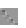}
&\hspace{-6ex}\includegraphics[trim = 0mm 0mm 0.5mm 0.5mm, clip, width=\iw]{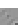}
&\hspace{-6ex}\includegraphics[trim = 0mm 0mm 0.5mm 0.5mm, clip, width=\iw]{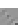}
&\hspace{-6ex}\includegraphics[trim = 0mm 0mm 0.5mm 0.5mm, clip, width=\iw]{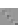}\\
 \hspace{-6ex}  \mbox{\small $\vp$}
&\hspace{-12ex} \mbox{\small $\varphi(\vx)=0$ }
&\hspace{-12ex} \mbox{\small $\varphi(\vx)=\vone^T\mB\vx$ }
&\hspace{-12ex} \mbox{\small $\varphi(\vx)=\ve^T\vx$ }\\
 \hspace{-6ex}  \mbox{\small Ground Truth }
&\hspace{-12ex} \mbox{\small 28.29 dB }
&\hspace{-12ex} \mbox{\small 28.50 dB }
&\hspace{-12ex} \mbox{\small 29.30 dB }
\end{tabular}
\caption{Denoising results: A ground truth patch cropped from an image, and the denoised patches of using different improvement schemes. Noise standard deviation is $\sigma = 50$. $\tau = 1/(2n)$ for $\varphi(\vx)=\vone^T\mB\vx$ and $\tau = 1$ for $\varphi(\vx)=\ve^T\vx$. }
\label{fig:denoised,improved_patch_matching}
\end{figure}

\begin{figure}[t]
\centering
\includegraphics[width=1\linewidth]{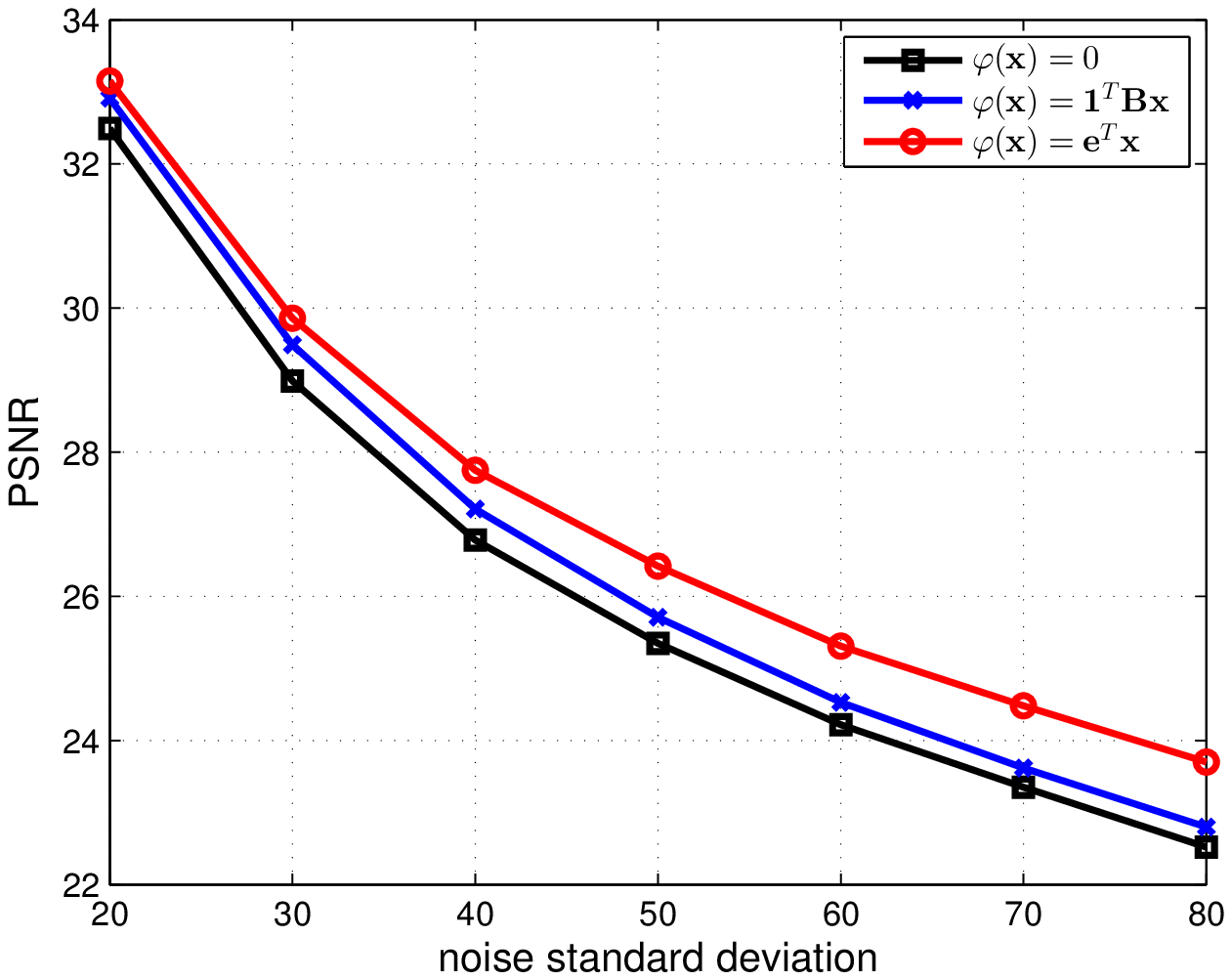}
\caption{Denoising results of three patch selection improvement schemes. The PSNR value is computed from a $432 \times 381$ image. }
\label{fig:noRefinement_crossSimilarity_basicEstimate_Cone}
\end{figure}

\subsubsection{Comparisons}
To demonstrate the effectiveness of the two proposed patch selection steps, we consider a ground truth (clean) patch shown in \fref{fig:improved_patch_matching} (a). From a pool of $n = 200$ reference patches, we apply an exhaustive search algorithm to choose $k = 40$ patches that best match with the noisy observation $\vq$, where the first 10 patches are shown in \fref{fig:improved_patch_matching} (b). The results of the two selection refinement methods are shown in \fref{fig:improved_patch_matching} (c)-(d), where in both cases the parameter $\tau$ is adjusted for the best performance. For the case of $\varphi(\vx) = \vone^T\mB\vx$, we set $\tau = 1/(200n)$ when $\sigma < 30$ and $\tau = 1/(2n)$ when $\sigma > 30$. For the case of $\varphi(\vx) = \ve^T\vx$, we use the denoised result of BM3D as the first-pass estimate $\vpbar$, and set $\tau = 0.01$ when $\sigma < 30$ and $\tau = 1$ when $\sigma > 30$. The results in \fref{fig:denoised,improved_patch_matching} show that the PSNR increases from 28.29 dB to 28.50 dB if we use $\varphi(\vx) = \vone^T\mB\vx$, and further increases to 29.30 dB if we use $\varphi(\vx) = \ve^T\vx$. The full performance comparison is shown in \fref{fig:noRefinement_crossSimilarity_basicEstimate_Cone}, where we show the PSNR curve for a range of noise levels of an image. Since the performance of $\varphi(\vx) = \ve^T\vx$ is consistently better than $\varphi(\vx) = \vone^T\mB\vx$, in the rest of the paper we focus on $\varphi(\vx) = \ve^T\vx$.

%% file: sparse_vs_group_sparse.tex
\begin{figure}[h]
\centering
\begin{tikzpicture}[scale=0.5]
\def\w{0.5}
\def\x{0}
\def\y{0}
\draw[black] (\x,\y) -- (\x, 10*\w);
\draw[step=\w cm,black] (\x,\y) grid ( \x  + \w , 10 * \w);
\draw[fill=black!50, opacity=0.5] (\x,1 * \w) rectangle (\x  + \w,2 * \w);
\draw[fill=black!50, opacity=0.5] (\x,3 * \w) rectangle (\x  + \w,4 * \w);
\draw[fill=black!50, opacity=0.5] (\x,6 * \w) rectangle (\x  + \w,7 * \w);

\def\x{1}
\def\y{0}
\draw[black] (\x,\y) -- (\x, 10*\w);
\draw[step=\w cm,black] (\x,\y) grid ( \x  + \w , 10 * \w);
\draw[fill=black!50, opacity=0.5] (\x,3 * \w) rectangle (\x + \w,4 * \w);
\draw[fill=black!50, opacity=0.5] (\x,7 * \w) rectangle (\x + \w,8 * \w);

\def\x{2}
\def\y{0}
\draw[black] (\x,\y) -- (\x, 10*\w);
\draw[step=\w cm,black] (\x,\y) grid ( \x  + \w , 10 * \w);
\draw[fill=black!50, opacity=0.5] (\x,2 * \w) rectangle (\x + \w,3 * \w);
\draw[fill=black!50, opacity=0.5] (\x,3 * \w) rectangle (\x + \w,4 * \w);
\draw[fill=black!50, opacity=0.5] (\x,7 * \w) rectangle (\x + \w,8 * \w);

\node at (\x+0.5,-1) {(a) sparse};

\def\x{3}
\def\y{0}
\draw[black] (\x,\y) -- (\x, 10*\w);
\draw[step=\w cm,black] (\x,\y) grid ( \x  + \w , 10 * \w);
\draw[fill=black!50, opacity=0.5] (\x,2 * \w) rectangle (\x + \w,3 * \w);
\draw[fill=black!50, opacity=0.5] (\x,6 * \w) rectangle (\x + \w,7 * \w);

\def\x{4}
\def\y{0}
\draw[black] (\x,\y) -- (\x, 10*\w);
\draw[step=\w cm,black] (\x,\y) grid ( \x  + \w , 10 * \w);
\draw[fill=black!50, opacity=0.5] (\x,3 * \w) rectangle (\x + \w,4 * \w);
\draw[fill=black!50, opacity=0.5] (\x,6 * \w) rectangle (\x + \w,7 * \w);

\def\x{5}
\def\y{0}
\draw[black] (\x,\y) -- (\x, 10*\w);
\draw[step=\w cm,black] (\x,\y) grid ( \x  + \w , 10 * \w);
\draw[fill=black!50, opacity=0.5] (\x,2 * \w) rectangle (\x + \w,3 * \w);
\draw[fill=black!50, opacity=0.5] (\x,3 * \w) rectangle (\x + \w,4 * \w);
\draw[fill=black!50, opacity=0.5] (\x,6 * \w) rectangle (\x + \w,7 * \w);

\def\offset{10}
\def\x{\offset + 0}
\def\y{0}
\draw[black] (\x,\y) -- (\x, 10*\w);
\draw[step=\w cm,black] (\x,\y) grid ( \x + \w , 10 * \w);
\draw[fill=black!50, opacity=0.5] (\x,2 * \w) rectangle (\x + \w,3 * \w);
\draw[fill=black!50, opacity=0.5] (\x,3 * \w) rectangle (\x + \w,4 * \w);
\draw[fill=black!50, opacity=0.5] (\x,6 * \w) rectangle (\x + \w,7 * \w);

\def\x{\offset + 1}
\def\y{0}
\draw[black] (\x,\y) -- (\x, 10*\w);
\draw[step=\w cm,black] (\x,\y) grid ( \x + \w , 10 * \w);
\draw[fill=black!50, opacity=0.5] (\x,2 * \w) rectangle (\x + \w,3 * \w);
\draw[fill=black!50, opacity=0.5] (\x,3 * \w) rectangle (\x + \w,4 * \w);
\draw[fill=black!50, opacity=0.5] (\x,6 * \w) rectangle (\x + \w,7 * \w);

\def\x{\offset + 2}
\def\y{0}
\draw[black] (\x,\y) -- (\x, 10*\w);
\draw[step=\w cm,black] (\x,\y) grid ( \x  + \w , 10 * \w);
\draw[fill=black!50, opacity=0.5] (\x,2 * \w) rectangle (\x + \w,3 * \w);
\draw[fill=black!50, opacity=0.5] (\x,3 * \w) rectangle (\x + \w,4 * \w);
\draw[fill=black!50, opacity=0.5] (\x,6 * \w) rectangle (\x + \w,7 * \w);

\node at (\x + 0.8,-1) {(b) group sparse};

\def\x{\offset + 3}
\def\y{0}
\draw[black] (\x,\y) -- (\x, 10*\w);
\draw[step=\w cm,black] (\x,\y) grid ( \x  + \w , 10 * \w);
\draw[fill=black!50, opacity=0.5] (\x,2 * \w) rectangle (\x + \w,3 * \w);
\draw[fill=black!50, opacity=0.5] (\x,3 * \w) rectangle (\x + \w,4 * \w);
\draw[fill=black!50, opacity=0.5] (\x,6 * \w) rectangle (\x + \w,7 * \w);

\def\x{\offset + 4}
\def\y{0}
\draw[black] (\x,\y) -- (\x, 10*\w);
\draw[step=\w cm,black] (\x,\y) grid ( \x  + \w , 10 * \w);
\draw[fill=black!50, opacity=0.5] (\x,2 * \w) rectangle (\x + \w,3 * \w);
\draw[fill=black!50, opacity=0.5] (\x,3 * \w) rectangle (\x + \w,4 * \w);
\draw[fill=black!50, opacity=0.5] (\x,6 * \w) rectangle (\x + \w,7 * \w);

\def\x{\offset + 5}
\def\y{0}
\draw[black] (\x,\y) -- (\x, 10*\w);
\draw[step=\w cm,black] (\x,\y) grid ( \x  + \w , 10 * \w);
\draw[fill=black!50, opacity=0.5] (\x,2 * \w) rectangle (\x + \w,3 * \w);
\draw[fill=black!50, opacity=0.5] (\x,3 * \w) rectangle (\x + \w,4 * \w);
\draw[fill=black!50, opacity=0.8] (\x,6 * \w) rectangle (\x + \w,7 * \w);
\end{tikzpicture}
\caption{Comparison between sparsity (where columns are sparse, but do not coordinate) and group sparsity (where all columns are sparse with similar locations).}
\label{fig:sparse vs group sparse}
\end{figure}

%% file: proposed_method_part2.tex
\section{Determine $\mLambda$}
\label{section4}
In this section, we present our proposed method to determine $\mLambda$ for a fixed $\mU$. Our proposed method is based on the concept of a Bayesian MSE estimator.


\subsection{Bayesian MSE Estimator}
Recall that the noisy patch is related to the latent clean patch as $\vq = \vp + \veta$, where $\veta \overset{\mathrm{\scriptsize{iid}}}{\sim} \calN(\vzero,\sigma^2\mI)$ denotes the noise. Therefore, the conditional distribution of $\vq$ given $\vp$ is
\begin{equation}
f(\vq \,|\, \vp) = \calN(\vp, \, \sigma^2\mI).
\end{equation}
Assuming that the prior distribution $f(\vp)$ is known, it is natural to consider the Bayesian mean squared error (BMSE) between the estimate $\vphat \bydef \mU\mLambda\mU^T\vq$ and the ground truth $\vp$:
\begin{equation}
\mathrm{BMSE} \bydef \E_{\vp}\left[ \E_{\vq|\vp}\left[ \left\| \vphat - \vp \right\|_2^2 \;\;\middle|\;\; \vp \right]\right].
\label{eq:bayesian MSE}
\end{equation}
Here, the subscripts remark the distributions under which the expectations are taken.

The BMSE defined in \eref{eq:bayesian MSE} suggests that the optimal $\mLambda$ should be the minimizer of the optimization problem
\begin{equation}
\mLambda = \argmin{\mLambda} \;\; \E_{\vp}\left[ \E_{\vq|\vp}\left[ \left\| \mU\mLambda\mU^T\vq - \vp \right\|_2^2 \;\;\middle|\;\; \vp \right]\right].
\label{eq:bayesian MSE Lambda}
\end{equation}
In the next subsection we discuss how to solve \eref{eq:bayesian MSE Lambda}.

\subsection{Localized Prior from the Targeted Database}
Minimizing BMSE over $\mLambda$ involves knowing the prior distribution $f(\vp)$. However, in general, the exact form of $f(\vp)$ is never known. This leads to many popular models in the literature, \emph{e.g.}, Gaussian mixture model \cite{Yu_Sapiro_Mallat_2012}, the field of expert model \cite{Roth_Black_2005, Roth_Black_2009}, and the expected patch log-likelihood model (EPLL) \cite{Zoran_Weiss_2011,Zoran_Weiss_2012}.

One common issue of all these models is that the prior $f(\vp)$ is built from a generic database of patches. In other words, the $f(\vp)$ models \emph{all} patches in the database. As a result, $f(\vp)$ is often a high dimensional distribution with complicated shapes.

In our targeted database setting, the difficult prior modeling becomes a much simpler task. The reason is that while the shape of the distribution $f(\vp)$ is still unknown, the subsampled reference patches (which are few but highly representative) could be well approximated as samples drawn from a single Gaussian centered around some mean $\vmu$ and covariance $\mSigma$. Therefore, by appropriately estimating $\vmu$ and $\mSigma$ of this \emph{localized} prior, we can derive the optimal $\mLambda$ as given by the following Lemma:

\begin{lemma}
\label{lemma:bayesian mse,solution}
Let $f(\vq \,|\, \vp) = \calN(\vp, \sigma^2\mI)$, and let $f(\vp) = \calN(\vmu, \mSigma)$ for any vector $\vmu$ and matrix $\mSigma$, then the optimal $\mLambda$ that minimizes \eref{eq:bayesian MSE} is
\begin{equation}
\label{eq:bayesian MSE, optimal solution}
\mLambda = \left(\mathrm{diag}(\mG+\sigma^2\mI)\right)^{-1} \mathrm{diag}(\mG),
\end{equation}
where $\mG \bydef \mU^T\vmu\vmu^T\mU+ \mU^T\mSigma\mU$.
\end{lemma}

\begin{proof}
See Appendix \ref{appendix:proof,bayesian mse,solution}.
\end{proof}

To specify $\vmu$ and $\mSigma$, we let
\begin{equation}
\label{eq:local prior estimators}
\vmu=\sum\limits_{j=1}^{k}w_j\vp_j, \quad \mSigma=\sum\limits_{j=1}^{k}w_j(\vp_j-\vmu)(\vp_j-\vmu)^T,
\end{equation}
where $w_j$ is the $j$th diagonal entry of $\mW$ defined in \eref{eq:W}. Intuitively, an interpretation of \eref{eq:local prior estimators} is that $\vmu$ is the non-local mean of the reference patches. However, the more important part of \eref{eq:local prior estimators} is $\mSigma$, which measures the \emph{uncertainty} of the reference patches with respect to $\vmu$. This uncertainty measure makes some fundamental improvements to existing methods which will be discussed in Section IV-C.

We note that Lemma~\ref{lemma:bayesian mse,solution} holds even if $f(\vp)$ is not Gaussian. In fact, for any distribution $f(\vp)$ with the first cumulant $\vmu$ and the second cumulant $\mSigma$, the optimal solution in \eref{eq:bayesian MSE, optimal solution} still holds. This result is equivalent to the classical linear minimum MSE (LMMSE) estimation \cite{Kay_1998}.

From a computational perspective, $\vmu$ and $\mSigma$ defined in  \eref{eq:local prior estimators} lead to a very efficient implementation as illustrated by the following lemma.
\begin{lemma}
\label{lemma:bayesian mse,solution with estimators}
Using $\vmu$ and $\mSigma$ defined in \eref{eq:local prior estimators}, the optimal $\mLambda$ is given by
\begin{equation}
\label{eq:solution for mLambda}
\mLambda = \left(\mathrm{diag}(\mS + \sigma^2\mI)\right)^{-1} \mathrm{diag}(\mS),
\end{equation}
where $\mS$ is the eigenvalue matrix of $\mP\mW\mP^T$.
\end{lemma}
\begin{proof}
See Appendix \ref{appendix:proof,bayesian mse,solution with estimators}.
\end{proof}

Combining Lemma \ref{lemma:bayesian mse,solution with estimators} with Lemma \ref{lemma:determine U,solution}, we observe that for any set of reference patches $\{\vp_j\}_{j=1}^k$, $\mU$ and $\mLambda$ can be determined \emph{simultaneously} through the eigen-decomposition of $\mP\mW\mP^T$. Therefore, we arrive at the overall algorithm shown in Algorithm \ref{alg:overalAlgorithm}.
\begin{algorithm}[!]
\caption{Proposed Algorithm}
\label{alg:overalAlgorithm}
\begin{algorithmic}
\STATE Input: Noisy patch $\vq$, noise variance $\sigma^2$, and clean reference patches $\vp_1,\ldots,\vp_k$
\STATE Output: Estimate $\vphat$
\STATE Learn $\mU$ and $\mLambda$
\begin{itemize}
\item Form data matrix $\mP$ and weight matrix $\mW$
\item Compute eigen-decomposition $[\mU,\mS] = \mbox{eig}(\mP\mW\mP^T)$
\item Compute $\mLambda = \left(\mathrm{diag}(\mS + \sigma^2\mI)\right)^{-1} \mathrm{diag}(\mS)$
\end{itemize}
\STATE Denoise: $\vphat = \mU\mLambda\mU^T \vq$.
\end{algorithmic}
\end{algorithm}

\subsection{Relationship to Prior Works}
It is interesting to note that many existing patch-based denoising algorithms assume some notions of prior, either explicitly or implicitly. In this subsection, we mention a few of the important ones. For notational simplicity, we will focus on the $i$th diagonal entry of $\mLambda = \diag{\lambda_1,\ldots,\lambda_d}$.

\subsubsection{BM3D \cite{Dabov_Foi_Katkovnik_2007}, Shape-Adaptive BM3D \cite{Dabov_Foi_Katkovnik_Egiazarian_2008} and BM3D-PCA \cite{Dabov_Foi_Katkovnik_Egiazarian_2009}
}
BM3D and its variants have two denoising steps. In the first step, the algorithm applies a basis matrix $\mU$ (either a pre-defined basis such as DCT, or a basis learned from PCA). Then, it applies a hard-thresholding to the projected coefficients to obtain a filtered image $\vpbar$. In the second step, the filtered image $\vpbar$ is used as a pilot estimate to the desired spectral component
\begin{equation}
\lambda_i = \frac{(\vu_i^T \vpbar)^2}{(\vu_i^T \vpbar)^2 + \sigma^2}.
\label{eq:bayesian,compare with bm3d}
\end{equation}

Following our proposed Bayesian framework, we observe that the role of using $\vpbar$ in \eref{eq:bayesian,compare with bm3d} is equivalent to assuming a dirac delta prior
\begin{equation}
f(\vp) = \delta(\vp - \vpbar).
\label{eq:bayesian,compare with bm3d, fp}
\end{equation}
In other words, the prior that BM3D assumes is concentrated at one point, $\vpbar$, and there is no measure of uncertainty. As a result, the algorithm becomes highly sensitive to the first-pass estimate. In contrast, \eref{eq:local prior estimators} suggests that the first-pass estimate can be defined as a non-local mean solution. Additionally, we incorporate a covariance matrix $\mSigma$ to measure the uncertainty of observing $\vmu$. These provide a more robust estimate to the denoising algorithm which is absent from BM3D and its variants.

\subsubsection{LPG-PCA \cite{Zhang_Dong_Zhang_2010}}
In LPG-PCA, the $i$th spectral component $\lambda_i$ is defined as
\begin{equation}
\label{eq:bayesian,compare with LPG-PCA}
\lambda_i = \frac{(\vu_i^T \vq)^2 - \sigma^2}{(\vu_i^T \vq)^2},
\end{equation}
where $\vq$ is the noisy patch. The (implicit) assumption in \cite{Zhang_Dong_Zhang_2010} is that $(\vu_i^T\vq)^2 \approx (\vu_i^T\vp)^2+\sigma^2$, and so by substituting $(\vu_i^T\vp)^2 \approx (\vu_i^T\vq)^2-\sigma^2$ into \eref{eq:mse,lambda_i} yields \eref{eq:bayesian,compare with LPG-PCA}. However, the assumption implies the existence of a perturbation $\Delta \vp$ such that $(\vu_i^T\vq)^2 = (\vu_i^T(\vp+\Delta\vp))^2+\sigma^2$. Letting $\vpbar = \vp+\Delta\vp$, we see that LPG-PCA implicitly assumes a dirac prior as in \eref{eq:bayesian,compare with bm3d} and \eref{eq:bayesian,compare with bm3d, fp}. The denoising result depends on the magnitude of $\Delta \vp$.

\subsubsection{Generic Global Prior \cite{Levin_Nadler_Durand_Freeman_2012}}
As a comparison to methods using generic databases such as \cite{Levin_Nadler_Durand_Freeman_2012}, we note that the key difference lies in the usage of a \emph{global} prior versus a \emph{local} prior. \fref{fig:generic_targeted_prior} illustrates the concept pictorially. The generic (global) prior $f(\vp)$ covers the entire space, whereas the targeted (local) prior is concentrated at its mean. The advantage of the local prior is that it allows one to denoise an image with few reference patches. It saves us from the intractable computation of learning the global prior, which is a high-dimensional non-parametric function.
\input{generic_targeted_prior}

\subsubsection{Generic Local Prior -- EPLL \cite{Zoran_Weiss_2011}, K-SVD \cite{Aharon_Elad_Bruckstein_2005, Elad_Aharon_2006}}
Compared to learning-based methods that use local priors, such as EPLL \cite{Zoran_Weiss_2011} and K-SVD \cite{Aharon_Elad_Bruckstein_2005, Elad_Aharon_2006}, the most important merit of the proposed method is that it requires significantly fewer training samples. A thorough justification will be discussed in Section V.

\subsubsection{PLOW \cite{Chatterjee_Milanfar_2012} }
PLOW has a similar design process as ours by considering the optimal filter. The major difference is that in PLOW, the denoising filter is derived from the full covariance matrices of the data and noise. As we will see in the next subsection, the linear denoising filter of our work is a truncated SVD matrix computed from a set of similar patches. The merit of the truncation is that it often reduces MSE in the bias-variance trade off \cite{Talebi_Milanfar_2014}.

\subsection{Improving $\mLambda$}
The Bayesian framework proposed above can be generalized to further improve the denoising performance. Referring to \eref{eq:bayesian MSE Lambda}, we observe that the BMSE optimization can be reformulated to incorporate a penalty term in $\mLambda$. Here, we consider the following $\ell_{\alpha}$ penalized BMSE:
\begin{equation}
\mathrm{BMSE}_{\alpha} \bydef \E_{\vp}\left[ \E_{\vq|\vp}\left[ \left\| \mU\mLambda\mU^T\vq - \vp \right\|_2^2 \middle| \vp \right]\right] + \gamma \|\mLambda \vec{1}\|_{\alpha},
\label{eq:penalized bayesian MSE}
\end{equation}
where $\gamma > 0$ is the penalty parameter, and $\alpha \in \{0,\,1\}$ controls which norm to be used. The solution to the minimization of \eref{eq:penalized bayesian MSE} is given by the following lemma.
\begin{lemma}
\label{lemma:penalized bayesian mse,solution}
Let $s_i$ be the $i$th diagonal entry in $\mS$, where $\mS$ is the eigenvalue matrix of $\mP\mW\mP^T$, then the optimal $\mLambda$ that minimizes $\mbox{BMSE}_{\alpha}$ is $\diag{\lambda_1,
\cdots,\lambda_d}$, where
\begin{align}
\lambda_i=\max\left({\frac{s_i-\gamma/2}{s_i+\sigma^2}, 0}\right), \quad\quad \mbox{for $\alpha=1$},
\label{penalized bayesian mse, l1 solution}\\
\lambda_i= \frac{s_i}{s_i+\sigma^2} \mathds{1} \left(\frac{s_i^2}{s_i+\sigma^2} > \gamma \right), \quad\quad \mbox{for $\alpha=0$}.
\label{penalized bayesian mse, l0 solution}
\end{align}
\end{lemma}
\begin{proof}
See Appendix \ref{appendix:proof,penalized bayesian mse,solution}.
\end{proof}

The motivation of introducing an $\ell_{\alpha}$-norm penalty in \eref{eq:penalized bayesian MSE} is related the group sparsity used in defining $\mU$. Recall from Section III that since $\mU$ is the optimal solution to a group sparsity optimization, only few of the entries in the ideal projection $\mU^T\vp$ should be non-zero. Consequently, it is desired to require $\mLambda$ to be sparse so that $\mU\mLambda\mU^T\vq$ has similar spectral components as that of $\vp$.

To demonstrate the effectiveness of the proposed $\ell_{\alpha}$ formulation, we consider the example patch shown in \fref{fig:denoised,improved_patch_matching}. For a refined database of $k = 40$ patches, we consider the original minimum BMSE solution ($\gamma = 0$), the $\ell_0$ solution with $\gamma = 0.02$, and the $\ell_1$ solution with $\gamma = 0.02$. The results in \fref{fig:original_l1_lo} show that with the proposed penalty term, the new $\mbox{BMSE}_{\alpha}$ solution performs consistently better than the original BMSE solution.

\begin{figure}[t]
\centering
\includegraphics[width=1\linewidth]{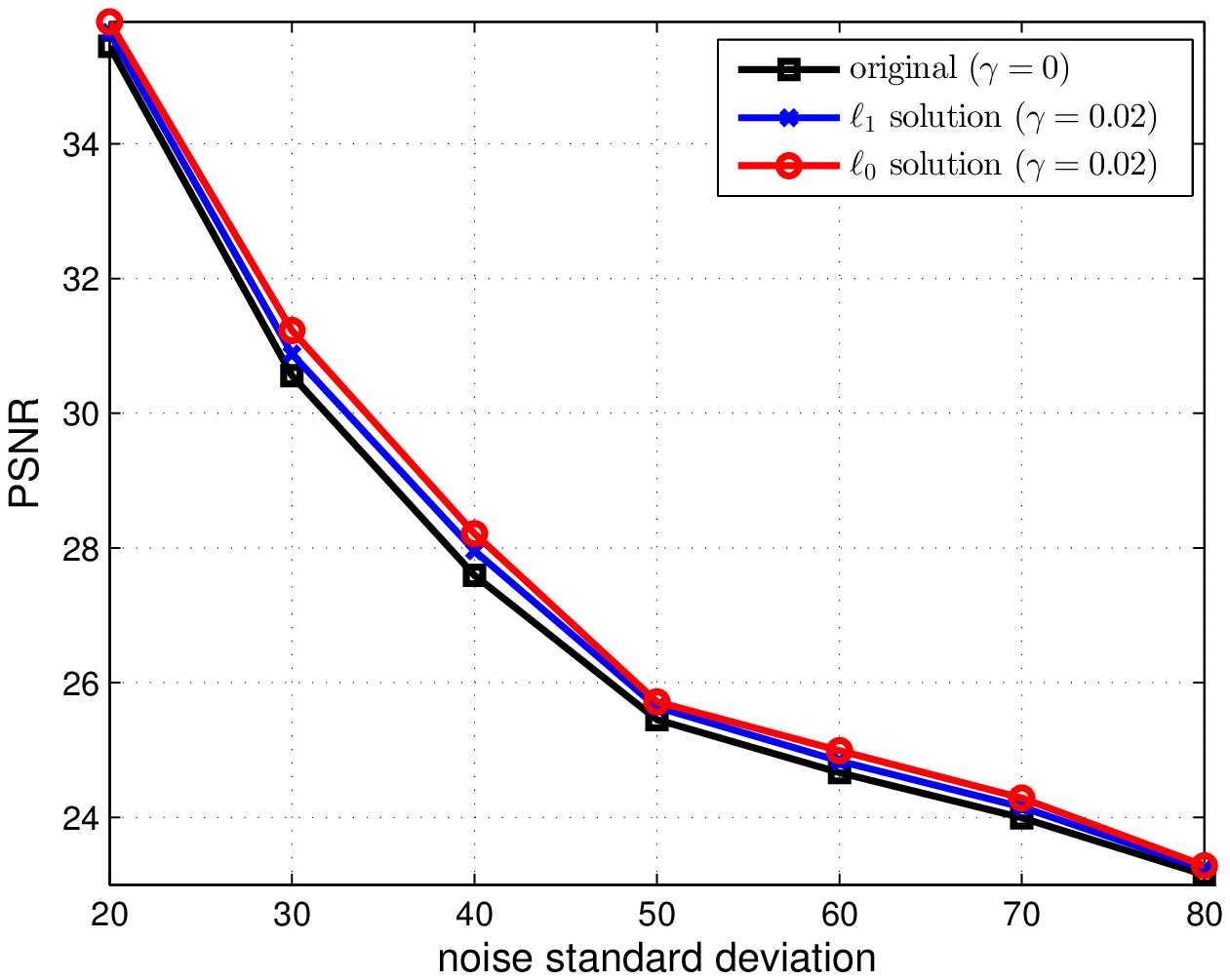}
\caption{Comparisons of the $\ell_1$ and $\ell_0$ adaptive solutions over the original solution with $\gamma = 0$. The PSNR value for each noise level is averaged over 100 independent trials to reduce the bias due to a particular noise realization.}
\label{fig:original_l1_lo}
\end{figure}

%% file: generic_targeted_prior.tex
\newcommand\gauss[2]{1/(#2*sqrt(2*pi))*exp(-((x-#1)^2)/(2*#2^2))}

\begin{figure}
\centering
\hspace{5ex}
\begin{tikzpicture}[thick, scale=1]
\begin{axis}[every axis plot post/.append style={
  mark=none,samples=100,smooth},
clip=false,
xmin = 0,
xmax = 16,
ymin = 0,
ymax = 1.5,
axis y line=left,
axis x line=bottom,
axis line style = very thick,
xtick=\empty,
ytick=\empty,
]
\addplot [mark = none, smooth, very thick,black!50!black, domain = 0.6:3.4] {\gauss{2}{0.4}};
\addplot [mark = none, smooth, very thick,black!50!black, domain = 4.5:7.5] {\gauss{6}{0.3}};
\draw [very thick](axis description cs:0.125,-0.03) -- (axis description cs:0.125,0.03);
\draw [very thick](axis description cs:0.375,-0.03) -- (axis description cs:0.375,0.03);
\node at (axis description cs:0.125,-0.08){$\boldsymbol{\mu}_1$};
\node at (axis description cs:0.375,-0.08){$\boldsymbol{\mu}_2$};

\node at (axis description cs:0.175,0.72){targeted $f_1(\boldsymbol{p})$};
\node at (axis description cs:0.395,0.92){targeted $f_2(\boldsymbol{p})$};
\node at (axis description cs:0.87,0.18){generic $f(\boldsymbol{p})$};
\addplot+[no markers, smooth, very thick,black!100!black] coordinates {
	(0.5,0.05) (2.5,0.5) (4,0.05) (6.5,0.6) (8,0.1) (10,0.2) (11,0.05) (12, 0.7)};
\end{axis}
\end{tikzpicture}
\caption{Generic prior vs targeted priors: Generic prior has an arbitrary shape spanned over the entire space; Targeted priors are concentrated at the means. In this figure, $f_1(\vp)$ and $f_2(\vp)$ illustrate two targeted priors which correspond to two patches of an image.}
\label{fig:generic_targeted_prior}
\end{figure}

%% file: experiments.tex
\section{Experimental Results}

\label{section5}
\begin{figure*}[th]
\centering
\def\iw{0.175\linewidth}
\begin{tabular}{ccccc}
\includegraphics[width=\iw]{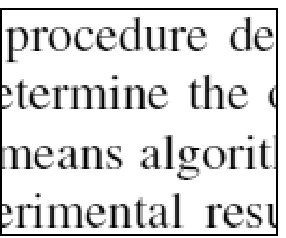}&
\includegraphics[width=\iw]{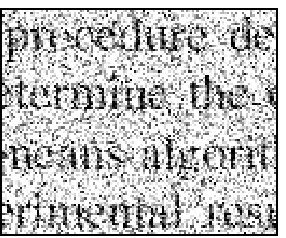}&
\includegraphics[width=\iw]{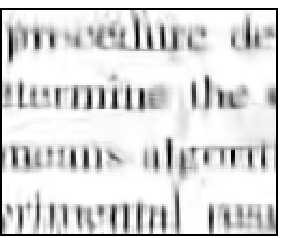}&
\includegraphics[width=\iw]{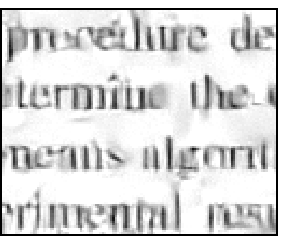}&
\includegraphics[width=\iw]{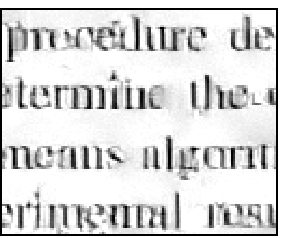}\\
(a) clean & (b) noisy $\sigma = 100$ & (c) iBM3D  & (d) EPLL(generic)  & (e) EPLL(target) \\
                     &                                       & 16.68 dB (0.7100)         &  16.93 dB (0.7341)                  &  18.65 dB (0.8234)\\
\includegraphics[width=\iw]{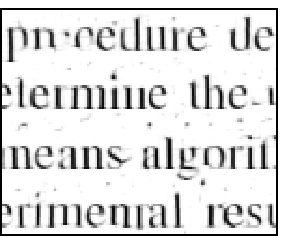}&
\includegraphics[width=\iw]{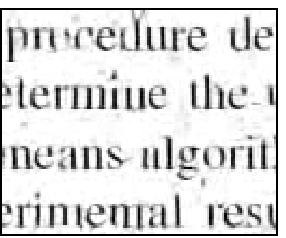}&
\includegraphics[width=\iw]{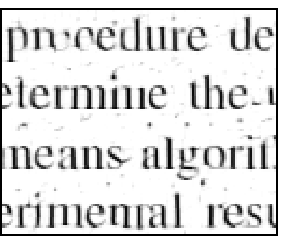}&
\includegraphics[width=\iw]{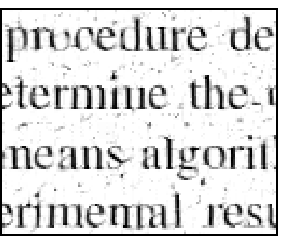}&
\includegraphics[width=\iw]{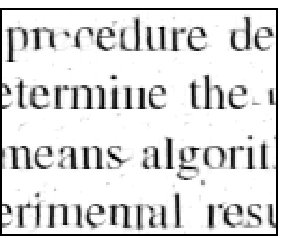}\\
(f) eNLM  & (g) eBM3D  & (h) eBM3D-PCA  & (i) eLPG-PCA  & (j) ours \\
 20.72 dB (0.8422)         &  20.33 dB (0.8228)          &  21.39 dB (0.8435)              &  20.37 dB (0.7299)             &  \bf{22.20} dB (\bf{0.9069})
\end{tabular}
\caption{Denoising text images: Visual comparison and objective comparison (PSNR and SSIM in the parenthesis). The test image size is of $127 \times 104$. Prefix ``\emph{i}'' stands for internal denoising (\emph{i.e.}, single-image denoising), and prefix ``\emph{e}'' stands for external denoising (\emph{i.e.}, using external databases).}
\label{fig:textDenoising}
\end{figure*}

In this section, we present a set of experimental results. 

\subsection{Comparison Methods}
The methods we choose for comparison are BM3D \cite{Dabov_Foi_Katkovnik_2007}, BM3D-PCA \cite{Dabov_Foi_Katkovnik_Egiazarian_2009}, LPG-PCA \cite{Zhang_Dong_Zhang_2010}, NLM \cite{Buades_Coll_2005_Journal}, EPLL \cite{Zoran_Weiss_2011} and KSVD \cite{Elad_Aharon_2006}. Except for EPLL and KSVD, all other four methods are internal denoising methods. We re-implement and modify the internal methods so that patch search is performed over the targeted external databases. These methods are iterated for two times where the solution of the first step is used as a basic estimate for the second step. The specific settings of each algorithm are as follows:

\begin{enumerate}
\item BM3D \cite{Dabov_Foi_Katkovnik_2007}: As a benchmark of internal denoising, we run the original BM3D code provided by the author\footnote{http://www.cs.tut.fi/\texttildelow{}foi/GCF-BM3D/}. Default parameters are used in the experiments, \emph{e.g.}, the search window is $39 \times 39$. We have included a discussion in Section \ref{experiment:text denoising} about the influence of different search window size to the denoising performance. As for external denoising, we implement an external version of BM3D. To ensure a fair comparison, we set the search window identical to other external denoising methods.

\item BM3D-PCA \cite{Dabov_Foi_Katkovnik_Egiazarian_2009} and LPG-PCA \cite{Zhang_Dong_Zhang_2010}: $\mU$ is learned from the best $k$ external patches, which is the same as in our proposed method. $\mLambda$ is computed following \eref{eq:bayesian,compare with bm3d} for BM3D-PCA and \eref{eq:bayesian,compare with LPG-PCA} for LPG-PCA. In BM3D-PCA's first step, the threshold is set to $2.7 \sigma$.

\item NLM \cite{Buades_Coll_2005_Journal}: The weights in NLM are computed according to a Gaussian function of the $\ell_2$ distance of two patches \cite{Buades_Coll_Morel, Luo_Pan_Nguyen_2012}. However, instead of using all reference patches in the database, we use the best $k$ patches following \cite{Kervrann_Boulanger_2007}.

\item EPLL \cite{Zoran_Weiss_2011}: In EPLL, the default patch prior is learned from a generic database (200,000 $8 \times 8$ patches). For a fair comparison, we train the prior distribution from our targeted databases using the same EM algorithm mentioned in \cite{Zoran_Weiss_2011}.

\item KSVD \cite{Elad_Aharon_2006}: In KSVD, two dictionaries are trained including a global dictionary and a targeted dictionary. The global dictionary is trained from a generic database of 100,000 $8 \times 8$ patches by the KSVD authors. The targeted dictionary is trained from a targeted database of 100,000 $8 \times 8$ patches containing similar content of the noisy image. Both dictionaries are of size $64 \times 256$.

\end{enumerate}
To emphasize the difference between the original algorithms (which are single-image denoising algorithms) and the corresponding new implementations for external databases, we denote the original, (single-image) denoising algorithms with ``$i$'' (internal), and the corresponding new implementations for external databases with ``$e$'' (external).

We add zero-mean Gaussian noise with standard deviations from $\sigma=20$ to $\sigma=80$ to the test images. The patch size is set as $8 \times 8$ ($i.e., d=64$), and the sliding step size is 6 in the first step and 4 in the second step. Two quality metrics, namely Peak Signal to Noise Ratio (PSNR) and Structural Similarity (SSIM) are used to evaluate the objective quality of the denoised images.

\begin{figure}[t]
\centering
\includegraphics[width=1\linewidth]{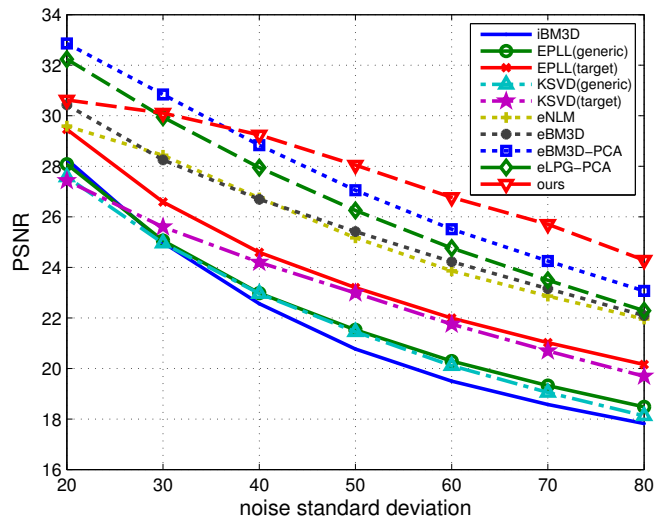}
\caption{Text image denoising: Average PSNR vs noise levels. In this plot, each PSNR value is averaged over 8 test images. The typical size of a test image is about $300 \times 200$.}
\label{fig:noiseStd_psnr_text}
\end{figure}

\subsection{Denoising Text and Documents}
\label{experiment:text denoising}
Our first experiment considers denoising a text image. The purpose is to simulate the case where we want to denoise a noisy document with the help of other similar but non-identical texts. This idea can be easily generalized to other scenarios such as handwritten signatures, bar codes and license plates.

To prepare this scenario, we capture randomly 8 regions of a document and add noise. We then build the targeted external database by cropping 9 arbitrary portions from a different document but with the same font sizes.

\subsubsection{Denoising Performance}
\fref{fig:textDenoising} shows the denoising results when we add excessive noise ($\sigma = 100$) to one query image. Among all the methods, the proposed method yields the highest PSNR and SSIM values. The PSNR is 5 dB better than the benchmark BM3D (internal) denoising algorithm. Some existing learning-based methods, such as EPLL, do not perform well due to the insufficient training samples from the targeted database. Compared to other external denoising methods, the proposed method shows a better utilization of the targeted database.

Since the default search window size for internal BM3D is only $39 \times 39$, we further conduct experiments to explore the effect of different search window sizes for BM3D. The PSNR results are shown in Table \ref{table:BM3D_different_region_size}. We see that a larger window size improves the BM3D denoising performance since more patch redundancy can be exploited. However, even if we extend the search to an external database (which is the case for eBM3D), the performance is still worse than the proposed method.

\begin{table}[h]
\centering
\begin{tabular}{|cc|c|c|c|}
\hline
& \small search window size & $\sigma=30$ & $\sigma=50$ & $\sigma=70$\\
\hline
\multirow{3}{*}{\small BM3D} & ($39 \times 39$) & 24.73 & 20.44 & 18.21\\
 & ($119 \times 119$) & 26.91 & 21.24 & 19.01\\
 & ($199 \times 199$) & 28.02 & 21.53 & 19.27\\
\small eBM3D & (\small external database) & 28.48 & 25.49 & 23.09\\
\small ours & (\small external database) & \bf{30.79} & \bf{28.43} & \bf{25.97}\\
\hline
\end{tabular}
\caption{PSNR results using BM3D with different search window sizes and the proposed method. We test the performance for three different noise levels ($\sigma=30,50,70$). The reported PSNR is computed on the entire image of size $301 \times 218$. }
\label{table:BM3D_different_region_size}
\end{table}

In \fref{fig:noiseStd_psnr_text}, we plot and compare the average PSNR values on 8 test images over a range of noise levels. We observe that at low noise levels ($\sigma < 30$), our proposed method performs worse than eBM3D-PCA and eLPG-PCA. One reason is that the patch variety of the text image database makes our estimate of $\mLambda$ in $\eref{eq:solution for mLambda}$ worse than the other two estimates in \eref{eq:bayesian,compare with bm3d} and \eref{eq:bayesian,compare with LPG-PCA}. However, as noise level increases, our proposed method outperforms other methods, which suggests that the prior of our method is more informative. For example, for $\sigma = 60$, our average PSNR result is 1.26 dB better than the second best result by eBM3D-PCA.

For the two learning-based methods, \emph{i.e.,} EPLL and KSVD, as can be seen, using a targeted database yields better results than using a generic database, which validates the usefulness of a targeted database. However, they perform worse than other non-learning methods. One reason is that a \emph{large} number of training samples are needed for these learning-based methods -- for EPLL, the large number of samples is needed to build the Gaussian mixtures, whereas for KSVD, the large number of samples is needed to train the over-complete dictionary. In contrast, the proposed method is fully functional even if the database is small.

\begin{figure}[t]
\centering
\includegraphics[width=1\linewidth]{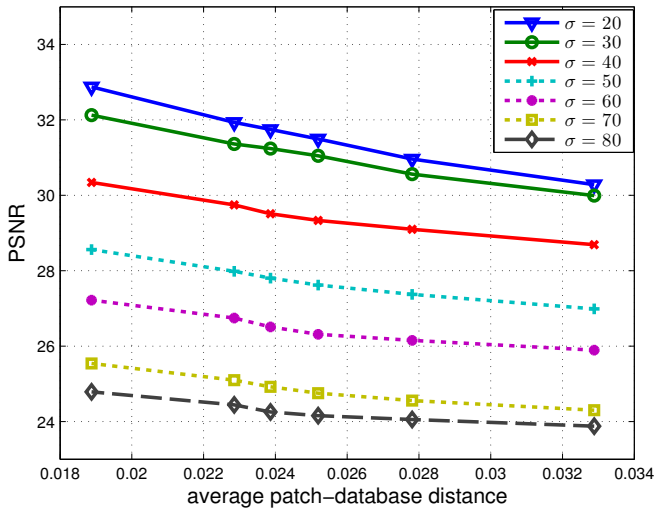}
\caption{Denoising performance in terms of the database quality. The average patch-to-database distance $\overline{d}(\calP)$ is a measure of the database quality.}
\label{fig:databaseQuality_vs_psnr}
\end{figure}

\subsubsection{Database Quality}
\begin{figure*}[th]
\centering
\def\iw{0.16\linewidth}
\begin{tabular}{cccccc}
\includegraphics[width=\iw]{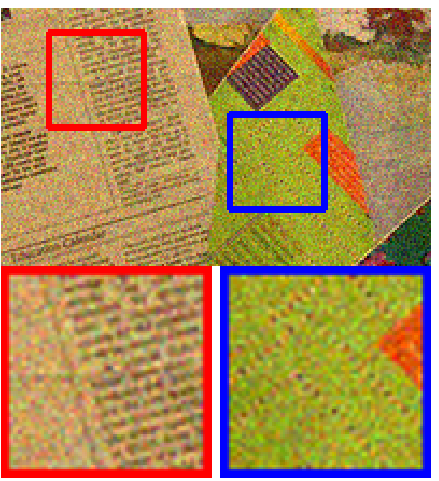}&
\hspace{-3ex}
\includegraphics[width=\iw]{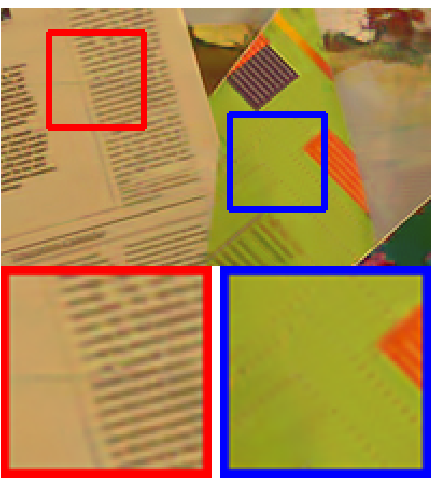}&
\hspace{-3ex}
\includegraphics[width=\iw]{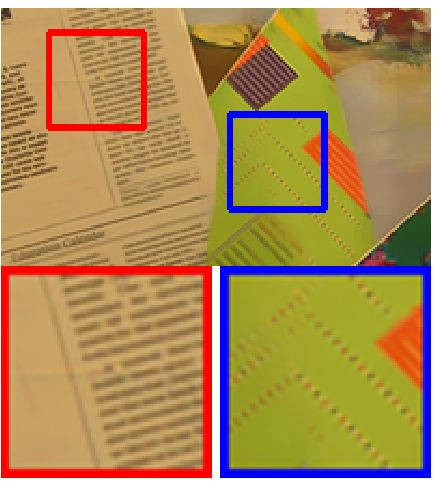} &
\hspace{-3ex}
\includegraphics[width=\iw]{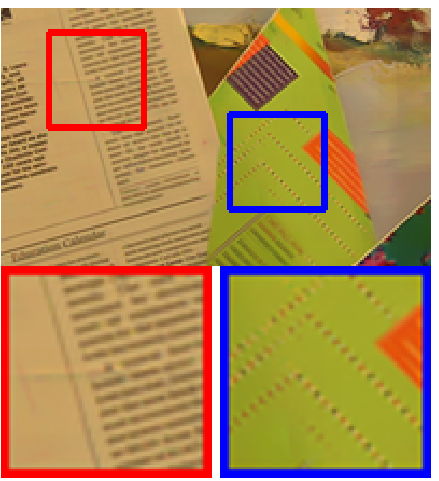} &
\hspace{-3ex}
\includegraphics[width=\iw]{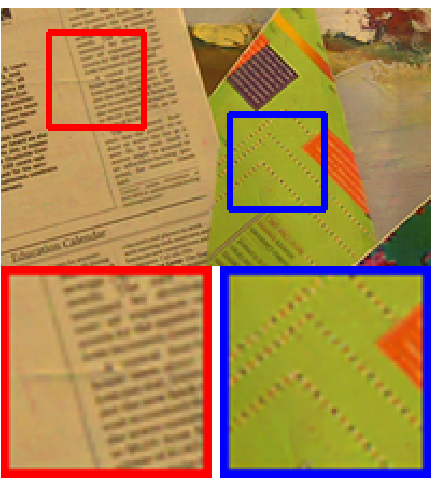} &
\hspace{-3ex}
\includegraphics[width=\iw]{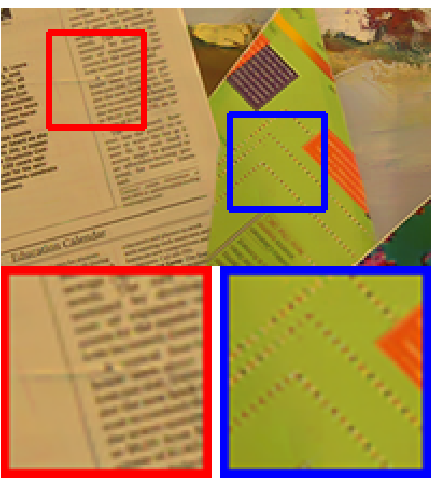}\\
\small (a) noisy  & \small (b) iBM3D
&  \small (c) eNLM & \small (d) eBM3D-PCA& \small (e) eLPG-PCA&  \small (f) ours\\
\small ($\sigma=20$) & \small 28.99 dB & \small 31.17 dB & \small 32.18 dB & \small 32.92 dB & \textbf{33.65 dB}\\
\includegraphics[width=\iw]{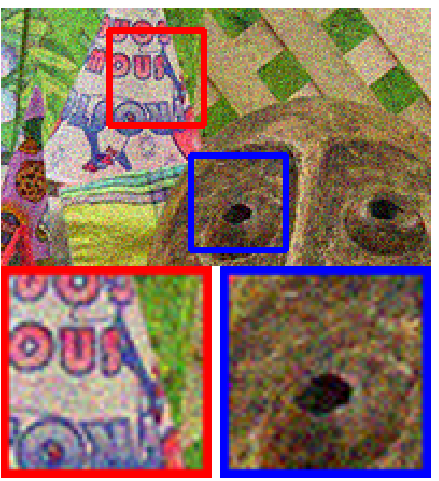}&
\hspace{-3ex}
\includegraphics[width=\iw]{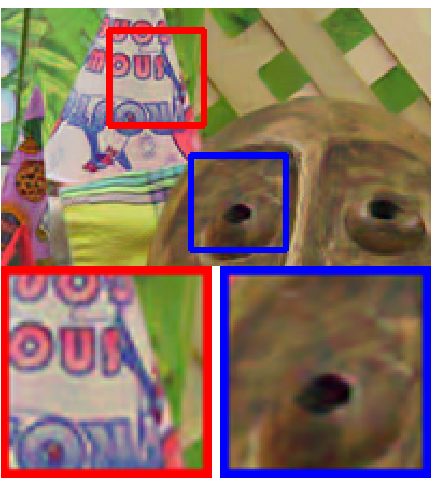}&
\hspace{-3ex}
\includegraphics[width=\iw]{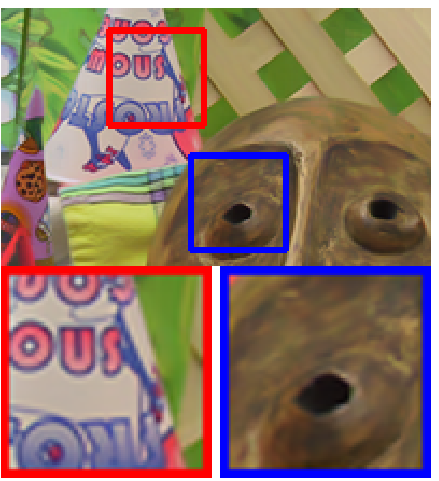} &
\hspace{-3ex}
\includegraphics[width=\iw]{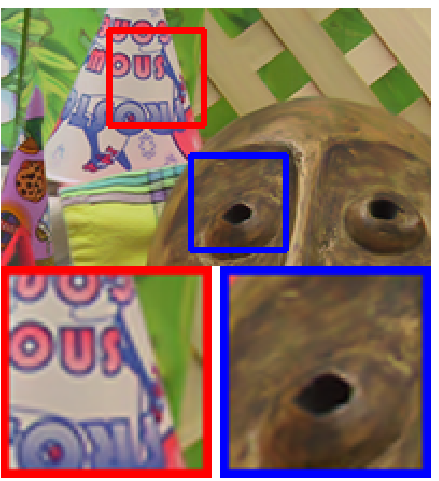} &
\hspace{-3ex}
\includegraphics[width=\iw]{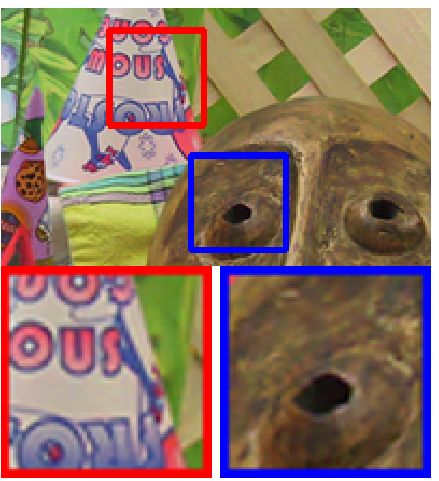} &
\hspace{-3ex}
\includegraphics[width=\iw]{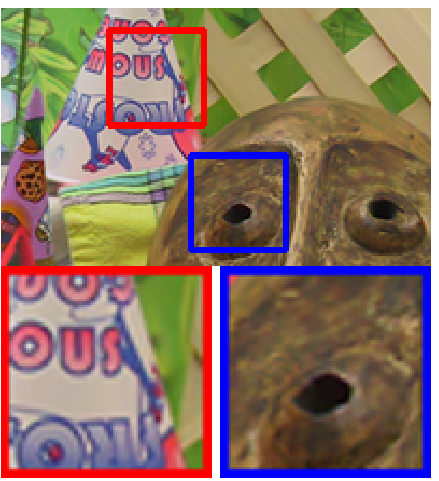}\\
\small (g) noisy & \small (h) iBM3D
&  \small (i) eNLM&  \small (j) eBM3D-PCA& \small (k) eLPG-PCA& \small (l) ours\\
\small ($\sigma=20$) & \small 28.77 dB & \small 29.97 dB & \small 31.03 dB & \small 31.80 dB & \small \textbf{32.18 dB}
\end{tabular}
\caption{Multiview image denoising: Visual comparison and objective comparison (PSNR). [Top] ``Barn''; [Bottom] ``Cone''. }
\label{fig:multiviewVisualResults}
\end{figure*}

We are interested in knowing how the quality of a database would affect the denoising performance, as that could offer us important insights about the sensitivity of the algorithm. To this end, we compute the average distance from a given database to a clean image that we would like to obtain. Specifically, for each patch $\vp_i \in \R^d$ in a clean image containing $m$ patches and a database $\calP$ of $n$ patches, we compute its minimum distance $$d(\vp_i, \calP) \bydef \min\limits_{\vp_j \in \calP} \, \|\vp_i-\vp_j\|_2/\sqrt{d}.$$ The average patch-database distance is then defined as $\overline{d}(\calP) \bydef (1/m) \sum_{i=1}^m d(\vp_i, \calP)$. Therefore, a smaller $\overline{d}(\calP)$ indicates that the database is more relevant to the ground truth (clean) image.

\fref{fig:databaseQuality_vs_psnr} shows the results of six databases $\calP$, where each is a random subset of the original targeted database. For all noise levels ($\sigma = $ 20 to 80), PSNR decreases linearly as the patch-to-database distance increase, Moreover, the decay rate is slower for higher noise levels. The result suggests that the quality of the database has a more significant impact under low noise conditions, and less under high noise conditions.

\subsection{Denoising Multiview Images}
Our second experiment considers the scenario of capturing images using a  multiview camera system. The multiview images are captured at different viewing positions. Suppose that one or more cameras are not functioning properly so that some images are corrupted with noise. Our goal is to demonstrate that with the help of the other clean views, the noisy view could be restored.

To simulate the experiment, we download 4 multivew datasets from Middlebury Computer Vision Page\footnote{http://vision.middlebury.edu/stereo/}. Each set of images consists of 5 views. We add i.i.d. Gaussian noise to one view and then use the rest 4 views to assist in denoising.

In \fref{fig:multiviewVisualResults}, we visually show the denoising results of the ``Barn'' and ``Cone'' multiview datasets. In comparison to the competing methods, our proposed method has the highest PSNR values. The magnified areas indicate that our proposed method removes the noise significantly and better reconstructs some fine details. In \fref{fig:multiviewDenoising}, we plot and compare the average PSNR values on 4 test images over a range of noise levels. The proposed method is consistently better than its competitors. For example, for $\sigma = 50$, our proposed method is  1.06 dB better than eBM3D-PCA and 2.73 dB better than iBM3D. The superior performance confirms our belief that with a good database, not any denoising algorithm would perform equally well. In fact, we still have to carefully design the denoising algorithm in order to maximize the performance by fully utilizing the database.

\begin{figure}[h]
\centering
\includegraphics[width=1\linewidth]{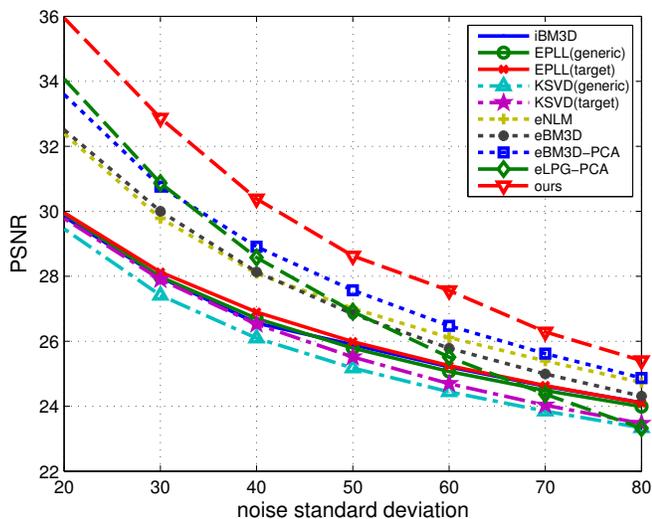}
\caption{Multiview image denoising: Average PSNR vs noise levels. In this plot, each PSNR value is averaged over 4 test images. The typical size of a test image is about $450 \times 350$.}
\label{fig:multiviewDenoising}
\end{figure}

\subsection{Denoising Human Faces}
Our third experiment considers denoising human face images. In low light conditions, images captured are typically corrupted by noise. To facilitate other high-level vision tasks such as recognition and tracking, denoising is a necessary pre-processing step. This experiment demonstrates the ability of denoising face images.

In this experiment, we use the Gore face database from \cite{Peng_Ganesh_Wright_Xu_Ma_2010}, of which some examples are shown in the top row of \fref{fig:faceVisualResults} (each image is $60 \times 80$). We simulate the denoising task by adding noise to 8 randomly chosen images and then use the other images (29 other face images in our experiment) in the database to assist in denoising.

In the top row of \fref{fig:faceVisualResults}, we show some clean face images in the database while in the bottom row, we show one of the noisy faces and its denoising results (magnified). We observe that while the facial expressions are different and there are misalignments between images, the proposed method still generates robust results. In \fref{fig:faceDenoising}, we plot the average PSNR curves on the 8 test images, where we see consistent gain compared to other methods.

\begin{figure}[h]
\vspace{-1ex}
\begin{minipage}[b]{1.0\columnwidth}
\centering
\includegraphics[width=1\textwidth]{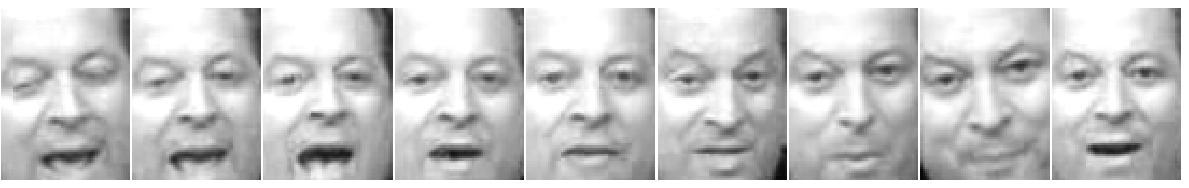}\\
\end{minipage}
\begin{minipage}[b]{1.0\columnwidth}
\vspace{-3ex}
\captionsetup[subfigure]{labelformat=empty}
\centering
\subfloat[][noisy\\ \centering ($\sigma=20$)]{\includegraphics[width=0.2\textwidth]{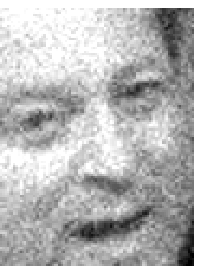}}
\subfloat[][\;\; iBM3D\\ \centering 32.04 dB]{\includegraphics[width=0.2\textwidth]{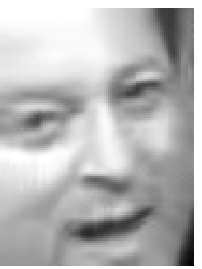}}
\subfloat[][eNLM\\ \centering 32.74 dB]{\includegraphics[width=0.2\textwidth]{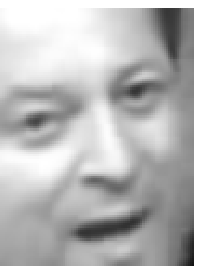}}
\subfloat[][ eBM3D-PCA \\ \centering 33.29 dB]
{\includegraphics[width=0.2\textwidth]{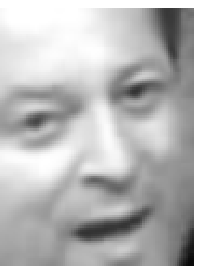}}
\subfloat[][ours\\ \centering \textbf{33.86 dB}]
{\includegraphics[width=0.2\textwidth]{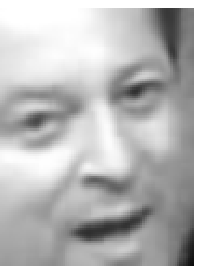}}
\end{minipage}
\caption{Face denoising of Gore dataset \cite{Peng_Ganesh_Wright_Xu_Ma_2010}. [Top] Database images; [Bottom] Denoising results.}
\label{fig:faceVisualResults}
\end{figure}

\begin{table*}[!]
\centering
\begin{tabular}{c|ccccc}
\hline
& iBM3D & EPLL(generic) & EPLL(target) & KSVD(generic) & KSVD(target)\\
runtime (sec) & 0.97 & 35.17 & 10.21 & 0.32 & 0.13\\
\hline
& eNLM & eBM3D &eBM3DPCA &eLPGPCA & ours\\
runtime (sec) & 95.68 & 99.17 & 102.21 & 102.14 & 144.33\\
\hline
\end{tabular}
\caption{Runtime comparison for different denoising methods. The test image is of size $301 \times 218$. For EPLL and KSVD methods, the time to train a finite Gaussian mixture model and the time to learn a dictionary is not included in the above runtime. For other external denoising methods, the targeted database consists of 9 images of similar sizes of the test image.}
\label{table:runtime_comparison}
\end{table*}

\begin{figure}[h]
\vspace{-1ex}
\centering
\includegraphics[width=1\linewidth]{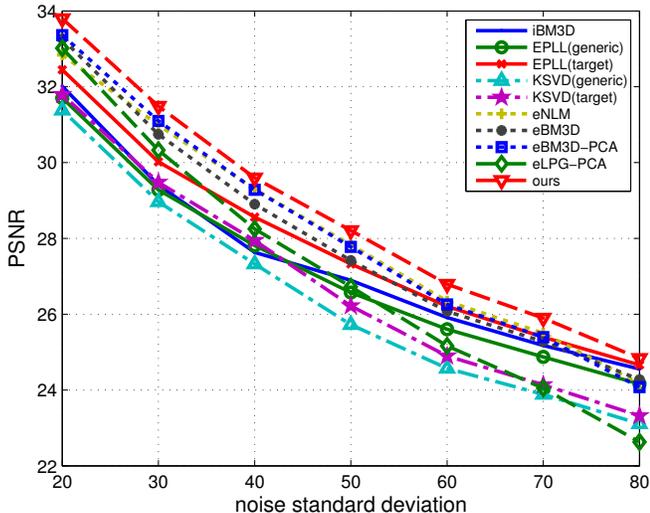}
\caption{Face denoising: Average PSNR vs noise levels. In this plot, each PSNR value is averaged over 8 test images. Each test image is of size $60 \times 80$.}
\label{fig:faceDenoising}
\end{figure}

\subsection{Runtime Comparison}
Our current implementation is in MATLAB (single thread). The runtime is about 144s to denoise an image ($301 \times 218$) with a targeted database consisting of 9 images of similar sizes. The code is run on an Intel Core i7-3770 CPU. In Table \ref{table:runtime_comparison}, we show a runtime comparison with other methods. We observe that the runtime of the proposed method is indeed not significantly worse than other external methods. In particular, the runtime of the proposed method is in the same order of magnitude as eNLM, eBM3D, eBM3D-PCA and eLPG-PCA.

We note that most of the runtime of the proposed method is spent on searching similar patches and computing SVD. Speed improvement for the proposed method is possible. First, we can apply techniques to enable fast patch search, e.g., patch match \cite{Mahmoudi_Sapiro_2005, Vignesh_Oh_Kuo_2010}, KD tree \cite{Muja_Lowe_2014}, or fast SVD \cite{Boutsidis_Magdon-Ismail_2014}. Second, random sampling schemes can be applied to further reduce the computational complexity \cite{Chan_Zickler_Lu_2013, Chan_Zickler_Lu_2014}. Third, since the denoising is independently performed on each patch, GPU can be used to parallelize the computation.

\subsection{Discussion and Future Work}
One important aspect of the algorithm that we did not discuss in depth is the sensitivity. In particular, two questions must be answered. First, assuming that there is a perturbation on the database, how much MSE will be changed? Answering the question will provide us information about the sensitivity of the algorithm when there are changes in font size (in the text example), view angle (in the multiview example), and facial expression (in the face example). Second, given a clean patch, how many patches do we need to put in the database in order to ensure that the clean patch is close to at least one of the patches in the database? The answer to this question will inform us about the size of the targeted database. Both problems will be studied in our future work.

%% file: conclusion.tex
\section{Conclusion}
Classical image denoising methods based on a single noisy input or generic databases are approaching their performance limits. We proposed an adaptive image denoising algorithm using targeted databases. The proposed method applies a group sparsity minimization and a localized prior to learn the basis matrix and the spectral coefficients of the optimal denoising filter, respectively. We show that the new method generalizes a number of existing patch-based denoising algorithms such as BM3D, BM3D-PCA, Shape-adaptive BM3D, LPG-PCA, and EPLL. Based on the new framework, we proposed improvement schemes, namely an improved patch selection procedure for determining the basis matrix and a penalized minimization for determining the spectral coefficients. For a variety of scenarios including text, multiview images and faces, we demonstrated empirically that the proposed method has superior performance over existing methods. With the increasing amount of image data available online, we anticipate that the proposed method is an important first step towards a data-dependent generation of denoising algorithms.

%% file: appendix.tex
\appendix
\section{Appendix}
\subsection{Proof of Lemma \ref{lemma:mse,solution}}
\label{appendix:proof,mse,solution}
\begin{proof}
From \eref{eq:mse}, the optimization is
\begin{equation*}
\begin{array}{ll}
\minimize{\vu_1,\ldots,\vu_d,\lambda_1,\ldots,\lambda_d} &\quad \sum_{i=1}^d \left[(1-\lambda_i)^2 (\vu_i^T\vp)^2 + \sigma^2 \lambda_i^2\right]\\
\quad\; \subjectto                                               &\quad \vu_i^T\vu_i = 1, \quad \vu_i^T\vu_j = 0.
\end{array}
\end{equation*}
Since each term in the sum of the objective function is non-negative, we can consider the minimization over each individual term separately. This gives
\begin{equation}
\begin{array}{ll}
\minimize{\vu_i,\lambda_i} &\quad (1-\lambda_i)^2 (\vu_i^T\vp)^2 + \sigma^2 \lambda_i^2\\
\subjectto                 &\quad \vu_i^T\vu_i = 1.
\end{array}
\label{eq:appendix1,step1}
\end{equation}
In \eref{eq:appendix1,step1}, we temporarily dropped the orthogonality constraint $\vu_i^T\vu_j = 0$, which will be taken into account later. The Lagrangian function of \eref{eq:appendix1,step1} is
\begin{equation*}
\calL(\vu_i, \lambda_i, \beta) = (1-\lambda_i)^2(\vu_i^T\vp)^2 + \sigma^2\lambda_i^2 + \beta(1-\vu_i^T\vu_i),
\end{equation*}
where $\beta$ is the Lagrange multiplier. Differentiating $\calL$ with respect to $\vu_i$, $\lambda_i$ and $\beta$ yields
\begin{align}
\frac{\partial \calL}{\partial \lambda_i} & = -2(1-\lambda_i)(\vu_i^T\vp)^2 + 2\sigma^2\lambda_i \label{eq:lemma,mse,proof,kkt,lambda}\\
\frac{\partial \calL}{\partial \vu_i}     & = 2(1-\lambda_i)^2(\vu_i^T\vp)\vp - 2\beta\vu_i \label{eq:lemma,mse,proof,kkt,vu}\\
\frac{\partial \calL}{\partial \beta}     & = 1 - \vu_i^T\vu_i. \label{eq:lemma,mse,proof,kkt,beta}
\end{align}
Setting $\partial \calL / \partial \lambda_i = 0$ yields
\begin{equation}
\lambda_i = (\vu_i^T\vp)^2/ \left((\vu_i^T\vp)^2+\sigma^2\right).
\label{eq:lambda,oracle case}
\end{equation}
Substituting this $\lambda_i$ into \eref{eq:lemma,mse,proof,kkt,vu} and setting $\partial \calL / \partial \vu_i = 0$ yields
\begin{equation}
\frac{2\sigma^4(\vu_i^T\vp)\vp}{\left((\vu_i^T\vp)^2+\sigma^2\right)^2} - 2\beta\vu_i = 0.
\label{eq:lemma,mse,proof,kkt,vu2}
\end{equation}
Therefore, the optimal pair ($\vu_i$, $\beta$) of \eref{eq:appendix1,step1} must be the solution of \eref{eq:lemma,mse,proof,kkt,vu2}. The corresponding $\lambda_i$ can be calculated via \eref{eq:lambda,oracle case}.

Referring to \eref{eq:lemma,mse,proof,kkt,vu2}, we observe two possible scenarios. First, if $\vu_i$ is any unit vector orthogonal to $\vp_i$, and $\beta=0$, then \eref{eq:lemma,mse,proof,kkt,vu2} can be satisfied. This is a trivial solution, because $\vu_i \bot \vp$ implies $\vu_i^T\vp=0$, and hence $\lambda_i = 0$. The second case is that
\begin{equation}
\vu_i = \vp/\|\vp\|_2, \quad\mbox{and}\quad \beta = \frac{\sigma^4\|\vp\|^2}{(\|\vp\|^2+\sigma^2)^2}.
\label{eq:lemma,mse,proof,vu,beta}
\end{equation}
Substituting \eref{eq:lemma,mse,proof,vu,beta} shows that \eref{eq:lemma,mse,proof,kkt,vu2} is satisfied. This is the non-trivial solution. The corresponding $\lambda_i$ in this case is $\|\vp\|^2 / (\|\vp\|^2 + \sigma^2)$.

Finally, taking into account of the orthogonality constraint $\vu_i^T\vu_j=0$ if $i \not= j$, we can choose $\vu_1 = \vp/\|\vp\|_2$, and $\vu_2 \bot \vu_1$, $\vu_3 \bot \{\vu_1, \vu_2\}$, $\ldots$, $\vu_d \bot \{\vu_1,\vu_2,\ldots \vu_{d-1}\}$. Therefore, the denoising result is
\begin{equation*}
\vphat = \mU
\left(\diag{\frac{\|\vp\|^2}{\|\vp\|^2+ \sigma^2},0,\ldots,0}\right)
\mU^T\vq,
\end{equation*}
where $\mU$ is any orthonormal matrix with the first column $\vu_1 = \vp/\|\vp\|_2$.
\end{proof}

\subsection{Proof of Lemma \ref{lemma:determine U,solution}}
\label{appendix:proof,determine U,solution}
\begin{proof}
Let $\vu_i$ be the $i$th column of $\mU$. Then,  \eref{eq:determine U,problem} becomes
\begin{equation}
\begin{array}{ll}
\minimize{\vu_1,\ldots,\vu_d}  &\quad \sum_{i=1}^d \|\vu_i^T\mP\|_{2} \\
\subjectto      &\quad \vu_i^T\vu_i = 1, \quad \vu_i^T\vu_j = 0.
\end{array}
\label{eq:determine U,lemma,eq1}
\end{equation}
Since each term in the sum of \eref{eq:determine U,lemma,eq1} is non-negative, we can consider each individual term
\begin{equation*}
\begin{array}{ll}
\minimize{\vu_i} &\quad \|\vu_i^T\mP\|_{2} \\
\subjectto       &\quad \vu_i^T\vu_i = 1,
\end{array}
\end{equation*}
which is equivalent to
\begin{equation}
\begin{array}{ll}
\minimize{\vu_i}&\quad \|\vu_i^T\mP\|_{2}^2 \\
\subjectto      &\quad \vu_i^T\vu_i = 1.
\end{array}
\label{eq:determine U,lemma,eq2}
\end{equation}

The constrained problem \eref{eq:determine U,lemma,eq2} can be solved by considering the Lagrange function,
\begin{equation}
\calL(\vu_i,\beta) = \|\vu_i^T\mP\|_{2}^2 + \beta(1-\vu_i^T\vu_i).
\end{equation}
Taking derivatives $\frac{\partial \calL}{\partial \vu_i} = 0$ and $\frac{\partial \calL}{\partial \beta} = 0$ yield
\begin{align*}
\mP\mP^T\vu_i = \beta\vu_i, \quad\mbox{and}\quad \vu_i^T\vu_i = 1.
\end{align*}
Therefore, $\vu_i$ is the eigenvector of $\mP\mP^T$, and $\beta$ is the corresponding eigenvalue. Since the eigenvectors are orthonormal to each other, the solution automatically satisfies the orthogonality constraint that $\vu_i^T\vu_j=0$ if $i \not= j$.
\end{proof}

\subsection{Proof of Lemma \ref{lemma:bayesian mse,solution}}
\label{appendix:proof,bayesian mse,solution}
\begin{proof}
First, by plugging $\vq=\vp+\veta$ into BMSE we get
\begin{align*}
\mathrm{BMSE} &= \E_{\vp}\left[ \E_{\vq|\vp}\left[ \left\| \mU\mLambda\mU^T (\vp +\veta) - \vp \right\|_2^2 \middle| \vp \right]\right]\\
&= \E_{\vp} \left[ \vp^T\mU\left(\mI-\mLambda\right)^2\mU^T\vp\right] + \sigma^2\mathrm{Tr}\left( \mLambda^2 \right).
\end{align*}
Recall the fact that for any random variable $\vx \sim \calN(\vmu_x, \mSigma_x)$ and any matrix $\mA$, it holds that $\E\left[ \vx^T\mA\vx\right] = \E[\vx]^T\mA\E[\vx]+\mathrm{Tr}\left(\mA\mSigma_x\right)$. Therefore, the above BMSE can be simplified as
\begin{align}
&\vmu^T\mU(\mI-\mLambda)^2\mU^T\vmu+\mathrm{Tr}\left(\mU(\mI-\mLambda)^2\mU^T\mSigma\right)+\sigma^2\mathrm{Tr}\left(\mLambda^2\right)\nonumber \\
=&\, \mathrm{Tr}\left((\mI-\mLambda)^2\mU^T\vmu\vmu^T\mU+(\mI-\mLambda)^2\mU^T\mSigma\mU\right)+\sigma^2\mathrm{Tr}\left(\mLambda^2\right) \nonumber \\
=&\, \mathrm{Tr}\left((\mI-\mLambda)^2\mG\right) + \sigma^2\mathrm{Tr}(\mLambda^2) \nonumber \\
=&\, \sum_{i=1}^{d}\left[(1-\lambda_i)^2 g_i + \sigma^2 \lambda_i^2 \right],
\label{simplified BMSE}
\end{align}
where $\mG \bydef \mU^T\vmu\vmu^T\mU + \mU^T\mSigma\mU$ and $g_i$ is the $i$th diagonal entry in $\mG$.

Setting $\partial \mathrm{BMSE} / \partial \lambda_i = 0$ yields
\begin{equation}
2(1-\lambda_i)g_i + 2\sigma^2\lambda_i = 0.
\end{equation}
Therefore, the optimal $\lambda_i$ is $g_i / (g_i+\sigma^2)$ and the optimal $\mLambda$ is
\begin{equation}
\label{eq:bayesian MSE, optimal solution}
\mLambda = \diag{\frac{g_1}{g_1+\sigma^2},\cdots,\frac{g_d}{g_d+\sigma^2}},
\end{equation}
which, by definition, is $\left(\mathrm{diag}(\mG+\sigma^2\mI)\right)^{-1}\mathrm{diag}(\mG)$.
\end{proof}

\subsection{Proof of Lemma \ref{lemma:bayesian mse,solution with estimators}}
\label{appendix:proof,bayesian mse,solution with estimators}
\begin{proof}
First, we write $\mSigma$ in \eref{eq:local prior estimators} in the matrix form
\begin{align*}
\mSigma =&\; \left(\mP-\vmu\vec{1}^T\right)\mW \left(\mP-\vmu\vec{1}^T\right)^T\\
=&\; \mP\mW\mP^T-\vmu\vec{1}^T\mW\mP^T-\mP\mW\vec{1}\vmu^T+\vmu\vec{1}^T\mW\vec{1}\vmu^T.
\end{align*}
It is not difficult to see that $\vec{1}^T\mW\mP^T = \vmu^T, \mP\mW\vec{1}=\vmu$ and $\vec{1}^T\mW\vec{1}=1$. Therefore,
\begin{equation*}
\mSigma =\; \mP\mW\mP^T - \vmu\vmu^T-\vmu\vmu^T+\vmu\vmu^T =\; \mP\mW\mP^T-\vmu\vmu^T,
\end{equation*}
which gives
\begin{equation}
\label{eq:prior variance estimator equivalence}
\vmu\vmu^T + \mSigma = \mP\mW\mP^T.
\end{equation}
Note that $\mG = \mU^T\vmu\vmu^T\mU + \mU^T\mSigma\mU = \mU^T(\vmu\vmu^T + \mSigma) \mU$. Substituting \eref{eq:prior variance estimator equivalence} into $\mG$ and using equation \eref{eq:determine U,solution}, we have
\begin{equation*}
\mG =\; \mU^T \mP\mW\mP^T \mU =\; \mU^T \mU \mS \mU^T \mU =\; \mS.
\end{equation*}
Therefore, by Lemma \ref{lemma:bayesian mse,solution},
\begin{equation}
\mLambda = \left(\mathrm{diag}(\mS+\sigma^2\mI)\right)^{-1}\mathrm{diag}(\mS).
\end{equation}
\end{proof}

\subsection{Proof of Lemma~\ref{lemma:penalized bayesian mse,solution}}
\label{appendix:proof,penalized bayesian mse,solution}
By Lemma \ref{lemma:bayesian mse,solution with estimators}, it holds that
\begin{align*}
&\E_{\vp}\left[ \E_{\vq|\vp}\left[ \left\| \mU\mLambda\mU^T \vq - \vp \right\|_2^2 \middle| \vp \right]\right]\\
=& \sum_{i=1}^{d}\left[(1-\lambda_i)^2 s_i + \sigma^2 \lambda_i^2 \right]\\
=& \sum_{i=1}^{d} \left[ (s_i + \sigma^2) \left(\lambda_i - \frac{s_i}{s_i+\sigma^2}\right)^2 + \frac{s_i\sigma^2}{s_i+\sigma^2} \right].
\end{align*}
Therefore, the minimization of \eref{eq:penalized bayesian MSE} becomes
\begin{equation}
\label{standard shrinkage problem}
\minimize{\lambda_i} \; \sum_{i=1}^{d}\left[ (s_i + \sigma^2) \left(\lambda_i - \frac{s_i}{s_i+\sigma^2}\right)^2 \right] + \gamma \| \Lambda \vone\|_{\alpha},
\end{equation}
where $\gamma \| \Lambda \vone\|_{\alpha} = \gamma \sum_{i=1}^d |\lambda_i|$ or $\gamma \sum_{i=1}^d \mathds{1}(\lambda_i \neq 0)$ for $\alpha = 1$ or $0$.  We note that when $\alpha =$ 1 or 0, \eref{standard shrinkage problem} is the standard shrinkage problem \cite{Li_2009}, in which a closed form solution exists. Following from \cite{Chan_Khoshabeh_Gibson_Gill_Nguyen_2011}, the solutions are given by
\begin{equation*}
\lambda_i = \max{\left(\frac{s_i-\gamma/2}{s_i+\sigma^2}, 0\right)}, \quad \quad \mbox{for }\alpha=1,
\end{equation*}
and
\begin{equation*}
\lambda_i = \frac{s_i}{s_i+\sigma^2} \mathds{1} \left(\frac{s_i^2}{s_i+\sigma^2} > \gamma \right), \quad \quad \mbox{for }\alpha=0.
\end{equation*}